\documentclass{amsart}

\usepackage{graphicx,fullpage}

\DeclareMathOperator*{\argmax}{argmax}

\newcommand{\diag}{\operatorname{diag}}

\newtheorem{theorem}{Theorem}
\newtheorem{lemma}{Lemma}
\newtheorem{corollary}{Corollary}
\newtheorem{proposition}{Proposition}

\newtheorem{remark}{Remark}
\title{Manifold learning with bi-stochastic kernels}

\author[]{Nicholas F. Marshall}
\address{Program in Applied Mathematics, Yale University, New Haven, CT 06511,
USA} 
\email{nicholas.marshall@yale.edu}

\author[]{Ronald R. Coifman}
\address{Program in Applied Mathematics, Yale University, New Haven, CT 06511,
USA} 
\email{ronald.coifman@yale.edu }


\begin{document}
\begin{abstract}
In this paper we answer the following question: what is the infinitesimal
generator of the diffusion process defined by a kernel that is normalized such
that it is bi-stochastic with respect to a specified measure? More precisely,
under the assumption that data is sampled from a Riemannian manifold we
determine how the resulting infinitesimal generator depends on the potentially
nonuniform distribution of the sample points, and the specified measure for the
bi-stochastic normalization. In a special case, we demonstrate a connection to
the heat kernel. We consider both the case where only a single data set is
given, and the case where a data set and a reference set are given. The spectral
theory of the constructed operators is studied, and Nystr\"om extension
formulas for the gradients of the eigenfunctions are computed. Applications to
discrete point sets and manifold learning are discussed.
\end{abstract}
\maketitle

\section{Introduction} \label{intro}
\subsection{Introduction}

In the fields of machine learning and data mining, kernel based methods related
to diffusion processes have proven to be effective tools for data analysis,
e.g. \cite{BerryHarlim2016, Cloninger2016,
KimDsilvaKevrekidisDebenedetti2015, LedermanTalmonWuLoCoifman2015,
MishneCohen2013}. Such methods are often studied as manifold learning problems
where the given data is assumed to have been sampled from an underlying
Riemannian manifold, see \cite{BelkinNiyogi2008, CoifmanLafon2006,
MarshallHirn2017, Singer2006}.  In this paper, we adopt this manifold
perspective to study bi-stochastic kernel constructions. Let $(X,d\mu)$ be a
measure space; we say that a positive kernel $b(x,y)$ is bi-stochastic with
respect to $d\mu$ if 
$$
\int_X b(x,z) d\mu(z) = \int_X b(z,y) d\mu(z) = 1 \quad \text{for all} \quad
x,y \in X.
$$
Bi-stochastic kernels can be constructed by appropriately normalizing positive
kernels. We are specifically interested in symmetric constructions. Suppose that
$X$ is a compact topological space, $d\mu$ is a positive Borel measure, and
$k(x,y)$ is a continuous positive symmetric kernel. Then it follows from a
result of Knopp and Sinkhorn \cite{KnoppSinkhorn1968} that there exists a 
positive continuous function $d(x)$ such that
$$
\frac{k(x,y)}{d(x) d(y)} \quad \text{is bi-stochastic with respect to $d\mu$}.
$$
Moreover, the bi-stochastic normalization of $k(x,y)$ can be determined
iteratively by alternating between normalizing the kernel by a function of $x$
such that the integral over $y$ is equal to $1$, and vice versa.  This iteration
was first studied in the matrix setting by Sinkhorn \cite{Sinkhorn1964}. In this
paper we consider bi-stochastic kernel constructions in two settings. First, in
Section \ref{mainresults}, we consider the continuous setting where $X$ is a
compact smooth Riemannian manifold. Here, we prove our main results, Theorems
\ref{thm1} and \ref{thm2}, which describe a connection between bi-stochastic
kernels and diffusion operators. Second, in Section \ref{discrete}, we consider
the discrete setting where $X$ is a finite set of points.  Here, we translate
the results of Theorems \ref{thm1} and \ref{thm2} into the discrete setting, and
present numerical examples. Finally, in Sections \ref{proofthm1} and
\ref{proofthm2} we prove Theorems \ref{thm1} and \ref{thm2}, respectively.

\section{Main Results} \label{mainresults}
\subsection{Single measure} 
Suppose that $X$ is a compact smooth Riemannian manifold endowed with a measure
$d\mu(x) = q(x) dx$, where $q \in C^3(X)$ is a positive density function with
respect to the volume measure $dx$ on $X$. Assume that $X$ is isometrically
embedded in $\mathbb{R}^d$ and define $k_\varepsilon : X \times X \to
\mathbb{R}$ by
$$
k_\varepsilon(x,y) := h \left( \frac{|x-y|^2}{\varepsilon} \right),
$$
where $h$ is a smooth function of exponential decay, and where $|\cdot|$ denotes
the Euclidean norm on $\mathbb{R}^d$. Given a positive weight function $w \in
C^3(X)$, we define $d\hat{\mu}(x) := d\mu(x)/w(x)$, and assert that there exists
a function $d \in C^\infty(X)$ such that the kernel $b_\varepsilon : X \times X
\rightarrow \mathbb{R}$ defined by
$$
b_\varepsilon(x,y) = \frac{k_\varepsilon(x,y)}{d(x) d(y)}
\quad \text{is bi-stochastic with respect to} \quad d\hat{\mu},
$$
i.e., $ \int_X b_\varepsilon(x,y) d\hat{\mu}(y)  =  \int_X b_\varepsilon(x,y)
d\hat{\mu}(x)  = 1$. The existence of such a function $d$ follows from
\cite{KnoppSinkhorn1968} and from the smoothness of $h$ as described in
Section \ref{existence}. Define the operator 
$$
B_\varepsilon : L^2(X,d\hat{\mu}) \to L^2(X,d\hat{\mu}) \quad \text{by} \quad
B_\varepsilon f (x) = \int_X b_\varepsilon(x,y) f(y) d\hat{\mu}(y).
$$
Our first result characterizes the infinitesimal generator associated with
$B_\varepsilon$, and requires a class of smooth test functions.  As in
\cite{CoifmanLafon2006}, we define $E_k$ to be the span of the first $k$ Neumann
eigenfunctions of the Laplace-Beltrami operator $\Delta$ on $X$, with the
convention (by the possibility of multiplying by $-1$) that $\Delta$ is positive
semi-definite.  Let
$$
m_0 = \int_{\mathbb{R}^d} h(|t|^2) dt \quad \text{and} \quad
m_2 = \int_{\mathbb{R}^d} t_1^2 h(|t|^2) dt,
$$
where $dt$ denotes the Lebesgue measure on $\mathbb{R}^d$.

\begin{theorem} \label{thm1}
Suppose $k > 0$ is fixed. Then on $E_k$
$$
\frac{I - B_{\varepsilon}}{\varepsilon} f \rightarrow \frac{m_2}{2 m_0} \left(
\frac{\Delta \left( f \left(\frac{q}{w}\right)^{1/2} \right)}{\left( \frac{q}{w}
\right)^{1/2}}  - 
\frac{\Delta \left( \frac{q}{w} \right)^{1/2} }{\left(
\frac{q}{w}\right)^{1/2}}  f \right) \quad \text{in $L^2$ norm as} \quad
\varepsilon \rightarrow 0,
$$
where $I$ denotes the identity operator.
\end{theorem}
\noindent For example, suppose that $h(|x|^2) = e^{-|x|^2}$ is the Gaussian kernel which
satisfies
$$
m_0 = \int_{\mathbb{R}^d}  e^{-|t|^2} dt = (2 \pi)^{d/2},
\quad \text{and} \quad
m_2 = \int_{\mathbb{R}^d} t_1^2 e^{-|t|^2} dt= 2 (2 \pi)^{d/2},
$$
and suppose that $w(x) = q_\varepsilon(x)$, where
$$
q_\varepsilon(x) := \int_X k_\varepsilon(x,y) d\mu(y).
$$
Then we have the following Corollary.

\begin{corollary} \label{cor1}
Let $k > 0$ be fixed. Suppose $h(|x|^2) = e^{-|x|^2}$ and $w(x) =
q_\varepsilon(x)$. Then on $E_k$
$$
\frac{I - B_{\varepsilon}}{\varepsilon} f   \rightarrow \Delta f
\quad \text{in $L^2$ norm as} \quad \varepsilon \rightarrow 0,
$$
where $\Delta$ denotes the positive semi-definite Laplace-Beltrami operator on
$X$.  
\end{corollary}

\noindent The proof of this Corollary is included after the proof of Theorem
\ref{thm1}; essentially, the proof amounts to checking that $w(x) =
q_\varepsilon(x)$ approximates $q(x)$ with sufficiently small error such that
$w(x)$ can be replaced by $q(x)$ in the result of Theorem \ref{thm1}. Corollary
\ref{cor1} suggests that the operator $B_\varepsilon$ can be viewed as an
approximation to the heat kernel; more precisely, we have the following result.

\begin{corollary} \label{cor2}
Suppose that $h(|x|^2) = e^{-|x|^2}$ and $w(x) = q_\varepsilon(x)$. Then on
$L^2(X)$
$$
B_{\varepsilon}^{t/\varepsilon} f \rightarrow e^{-t \Delta} f
\quad \text{in $L^2$ norm as} \quad \varepsilon \rightarrow 0,
$$
where $e^{-\Delta t}$ denotes the Neumann heat kernel on $X$.
\end{corollary}

\noindent The proof of Corollary \ref{cor2} is included following the proof of Theorem
\ref{thm1}.

\subsection{Connection to diffusion maps}

The reader might notice the similarity between the bi-stochastic operator
$B_{\varepsilon}$ and the averaging operator $A_\varepsilon$ of Lafon
\cite{Lafon2004} which also approximates the Laplace-Beltrami operator.
Specifically, the operator  $A_{\varepsilon} : L^2(X,dx) \rightarrow L^2(X,dx)$
is defined by 
$$
A_{\varepsilon} f(x) = \frac{1}{v(x)^2} \int_X
\frac{k_\varepsilon(x,y)}{q_\varepsilon(x) q_\varepsilon(y)} f(y)
d\mu(y),
\quad
\text{where}
\quad
v(x)^2 = \int_X
\frac{k_\varepsilon(x,y)}{q_\varepsilon(x) q_\varepsilon(y)}
d\mu(y).
$$
Unlike the operator $B_\varepsilon$, the operator $A_\varepsilon$ is only
row stochastic, not bi-stochastic. Thus, we argue that $B_\varepsilon$ provides
a more natural analog to the heat kernel than the averaging operator
$A_\varepsilon$. Of course, when these constructions are performed on data,
numerically these operators are very similar (they have the same limiting
behavior) so in the case of a single data set, the difference between
$B_\varepsilon$ and $A_\varepsilon$ is mostly a matter of perspective. However,
this bi-stochastic perspective is valuable in that it translates to more general
settings as we demonstrate in the following section, where a reference measure
is additionally given. We remark that bi-stochastic kernels have been previously
considered in the context of reference graphs by Coifman and Hirn
\cite{CoifmanHirn2013}; however, in this paper a specialized construction is
presented which avoids the use of Sinkhorn iteration, and infinitesimal
generators are not investigated.

%

\subsection{Reference measure} 
We are motivated by the reference set embedding of Haddad, Kushnir, and Coifman
\cite{HaddadKushnirCoifman2014,KushnirHaddadCoifman2014} . Suppose that in
addition to the measure space $(X,d\mu)$, a reference measure space 
$$
(X,d\nu) \quad \text{is given such that} \quad d\nu(r) =
p(r) dr,
$$ 
where $p \in C^3(X)$ is a positive density function with respect to the
volume measure $dr$ on $X$. Define $d\hat{\nu}(r) = d\nu(r)/v(r)$ where $v
\in C^3(X)$ is a given positive weight function. We assert that there exists a
positive function $d \in C^\infty(X)$ such that the kernel $c_\varepsilon : X
\times X \rightarrow \mathbb{R}$ defined by
$$
c_\varepsilon(x,y) = \frac{\int_X k_\varepsilon(x,r) k_\varepsilon(y,r)
d\hat{\nu}(r)}{d(x) d(y)} \quad \text{is bi-stochastic with respect to}
\quad
d\hat{\mu},
$$
i.e., $\int_X c_\varepsilon(x,y) d\hat{\mu}(y) = \int_X c_\varepsilon(x,y)
d\hat{\mu}(x) =1$. As before, the existence of such a function $d$ follows from
\cite{KnoppSinkhorn1968} and the smoothness of $c_\varepsilon$ as
discussed in Section \ref{existence}. Define the operator 
$$
C_{\varepsilon} : L^2(X,d\hat{\mu}) \rightarrow L^2(X,d\hat{\mu}) 
\quad \text{by} \quad
C_{\varepsilon} f (x) = \int_X c_{\varepsilon}(x,y) f(y) d\hat{\mu}(y).
$$ 
The following result concerns the infinitesimal generator associated with
$C_\varepsilon$. We consider the same class of smooth test function $E_k$ as
before; that is, $E_k$ is the span of the first $k$ Neumann eigenfunctions of
the Laplace-Beltrami operator $\Delta$ on $X$.

\begin{theorem} \label{thm2}
Suppose $k > 0$ is fixed. Then on $E_k$
$$
\frac{I - C_{\varepsilon}}{\varepsilon} f \to \frac{m_2}{2 m_0} \left( \frac{
\Delta \left( f \cdot \left( \frac{q \cdot v}{w \cdot p} \right)^{1/2}
\right)}{\left( \frac{q \cdot v}{w \cdot p} \right)^{1/2}} - \frac{ \Delta
\left( \frac{q \cdot v}{w \cdot p} \right)^{1/2} }{\left( \frac{q \cdot v}{w
\cdot p} \right)^{1/2}} \cdot f + \frac{ \Delta \left( f \cdot \left( \frac{q
\cdot p}{w \cdot v} \right)^{1/2} \right)}{\left( \frac{q \cdot p}{w \cdot v}
\right)^{1/2}} - \frac{ \Delta \left( \frac{q \cdot p}{w \cdot v} \right)^{1/2}
}{\left( \frac{q \cdot p}{w \cdot v} \right)^{1/2}} \cdot f \right)
$$
in $L^2$ norm as $\varepsilon \rightarrow \infty$.
\end{theorem}

\noindent For example, suppose that  $h(|x|^2) = e^{-|x|^2}$ is the Gaussian
kernel, and define two density estimates
$$
q_\varepsilon(x) = \int_X k_\varepsilon(x,y) q(y) dy \quad 
\text{and} \quad p_\varepsilon(r) = \int_X k_\varepsilon(r,s) p(s) ds.
$$
Our first Corollary concerns the case where the measure $q(x)$ is of interest,
while the measure $p(r)$ is not. For example, in the discrete setting discussed
in Section \ref{discrete} the density $p(r)$ may be artificial, while the
density $q(x)$ may describe clusters in the data.

\begin{corollary} \label{cor3}
Suppose $k > 0$ is fixed. If $h(|x|^2)=e^{-|x|^2}$, $w(x) = 1/q_\varepsilon(x)$
and $v(r) = p_\varepsilon(r)$, then on $E_k$
$$
\frac{I - C_\varepsilon}{\varepsilon} f \rightarrow \frac{\Delta \left( f
q \right)}{q} - \frac{\Delta q}{q} f \quad
\text{in $L^2$ norm as } \varepsilon \to 0.
$$
\end{corollary}

\noindent The proof of this Corollary is immediate from Theorem \ref{thm2} and
the proof of Corollary \ref{cor1}. If neither density function is of interest,
then their effects can be removed from the associated infinitesimal generator as
in Corollary \ref{cor1}.

\begin{corollary} \label{cor4}
Let $k > 0$ be fixed. If $h(|x|^2)=e^{-|x|^2}$, $w(x) = q_\varepsilon(x)$,
and $v(r) = p_\varepsilon(r)$, then on $E_k$
$$
 \frac{I - C_{\varepsilon}}{\varepsilon} f \to
\Delta f \quad \text{in $L^2$ norm as } \varepsilon \rightarrow 0,
$$
where $\Delta$ denotes the positive semi-definite Laplace-Beltrami operator on
$X$.  
\end{corollary}

\noindent The proof of this Corollary is immediate from Theorem \ref{thm2} and
the error analysis performed in the proof of Corollary \ref{cor1}.

\begin{corollary} \label{cor5}
Suppose $h(|x|^2) = e^{-|x|^2}$, $w(x) = q_\varepsilon(x)$, and $v(r) =
p_\varepsilon(r)$. Then on $L^2(X)$
$$
C_\varepsilon^{t/\varepsilon} f \to e^{-t \Delta} f \quad
\text{in $L^2$ norm as } \varepsilon \to 0,
$$
where $e^{-t \Delta}$ is the Neumann heat kernel on $X$.
\end{corollary}

\noindent The proof of this Corollary is immediate from Corollary \ref{cor4} and the error
analysis in the proof of Corollary \ref{cor2}.

\subsection{Eigenfunctions and their gradients}
The idea of using eigenfunctions of the heat kernel as coordinates for
Riemannian manifolds originates with Berard, Besson, and Gallot
\cite{BerardBessonGallot1994}, and originates in the context of data analysis
with Coifman and Lafon \cite{CoifmanLafon2006}. Subsequently, Jones, Maggioni,
and Schul \cite{JonesMaggioniSchul2008} proved that the eigenfunctions of the
heat kernel provide locally bi-Lipschitz coordinates, and moreover, that these
eigenfunctions and their gradients satisfy estimates from above and below.
Informally speaking, the gradient of an eigenfunction describes how the given
eigenfunction views a neighborhood of its domain; in fact, in
\cite{JonesMaggioniSchul2008} the magnitude of the gradient of eigenfunctions is
used as a selection criteria to build local coordinates. Moreover, the gradient
of eigenfunctions is generally a useful tool for data analysis; for example,
Wolf, Averbuch, and Neittaanm\"aki \cite{WolfAverbuchNeittaanmaki2014}, and Wolf
\cite{Wolf2009} used the gradients of eigenfunctions for characterization and
detection of anomalies. In the following, we justify the existence of the
spectral decomposition of $B_\varepsilon$ and $C_\varepsilon$. Furthermore,
under the assumption that both operators are based on the Gaussian kernel we
compute Nystr\"om extension type formulas for the gradients of the
eigenfunctions of $B_\varepsilon$ and $C_\varepsilon$.  Recall, that the
bi-stochastic operator $B_{\varepsilon} : L^2(X,d\hat{\mu}) \rightarrow
L^2(X,d\hat{\mu})$ is defined 
$$
(B_{\varepsilon} f)(x) = \int_X b_{\varepsilon}(x,y) f(y) d\hat{\mu}(y).
$$
Since the kernel $b_\varepsilon$ is bounded, and $X$ is compact, it
follows that the operator $B_\varepsilon$ is a Hilbert-Schmidt operator, and
hence a compact operator.  Moreover, since $b_\varepsilon$ is also symmetric it
follows that the operator $B_\varepsilon$ is self-adjoint and we may apply the
Spectral Theorem to conclude that
$$
(B_\varepsilon f)(x) = \sum_{k=0}^\infty \lambda_k \langle f, \varphi_k
\rangle_{L^2(X,d\hat{\mu})} \varphi_k(x),
$$
where $1 = \lambda_0 > \lambda_1 \ge \lambda_2 \ge \cdots$ are the eigenvalues of
$B_\varepsilon$, and $\varphi_0,\varphi_1,\varphi_2,\ldots,$ are the
corresponding orthonormal eigenfunctions. In the following, we show that the
gradients of the eigenfunctions of $B_\varepsilon$ have a convenient formulation
when the construction is based on the Gaussian kernel.

\begin{proposition} \label{grad1}
Suppose $h(|x|^2) = e^{-|x|^2}$. Then
$$
\nabla \varphi_k(x) = \frac{1}{\lambda_k} \int_X \frac{y -
\bar{x}}{\varepsilon} b(x,y) \varphi_k(y) d\hat{\mu}(y)
\quad
\text{where}
\quad
\bar{x} := \int_X y b(x,y) d\hat{\mu}(y).
$$
\end{proposition}

\noindent In the absence of a manifold assumption the formula for $\nabla \varphi_k(x)$ in
Proposition \ref{grad1} can be taken as a definition; for example, in Section
\ref{discrete}, we discuss how Proposition \ref{grad1} can be used to define the
gradient of an eigenfunction in a setting where $X$ is a discrete set without
any given differential structure.

\subsection{Reference measure eigenfunctions and their gradients}
Similar results hold for the operator $C_\varepsilon :
L^2(X,d\hat{\mu}) \rightarrow L^2(X,d\hat{\mu})$ defined by
$$
C_\varepsilon f (x) = \int_X c_\varepsilon(x,y) f(y) d\hat{\mu}(y).
$$
Since the kernel $c_\varepsilon$ is bounded and $X$ is compact, clearly
$C_\varepsilon$ is a Hilbert-Schmidt operator, which implies that
$C_\varepsilon$ is a compact operator. Moreover, the symmetry of $c_\varepsilon$
implies that $C_\varepsilon$ is self-adjoint and thus by the Spectral Theorem
$$
(C_\varepsilon f)(x) = \sum_{j=0}^\infty \lambda_j \langle f, \varphi_j
\rangle_{L^2(X,d\hat{\mu})} \varphi_j(x),
$$
where $1 = \lambda_0 > \lambda_1 \ge \lambda_2 \ge \cdots$ are the eigenvalues,
and $\varphi_0,\varphi_1,\varphi_2,\ldots$ are the corresponding orthonormal
eigenfunctions on $L^2(X,d\hat{\mu})$. In this case, the Nystr\"om extension
type formula for the gradient of an eigenfunction is slightly more involved,
but similar in spirit to the case of a single measure.

\begin{proposition} \label{grad2}
Suppose $h(|x|^2) = e^{-|x|^2}$. Then
$$
\nabla_x \varphi_j(x) = \frac{1}{\lambda_k} \int_X 
\frac{r - \bar{r}_x}{\varepsilon}  (F_\varepsilon \varphi_j
)(x,r) d\hat{\nu}(r).
$$
where
$$
(F_\varepsilon f) (x,r) = \int_X \frac{k_\varepsilon(x,r)
k_\varepsilon(y,r)}{d(x) d(y)} f(y)
d\hat{\mu}(y) \quad \text{and} \quad \bar{r}_x := \int_X r
 (F_\varepsilon 1) (x,r)d\hat{\nu}(r). 
$$
\end{proposition}

\noindent While this formula for the gradient of an eigenfunction is less elegant than the
result from Proposition \ref{grad1}, it nevertheless provides a method of
defining $\nabla \varphi_j(x)$ in an arbitrary setting.

\section{The discrete setting} \label{discrete}
\subsection{Single measure}
Suppose that a data set of $n$ points $X = \{x_i\}_{i=1}^n \subset \mathbb{R}^d$
is given. The data set $X$ induces a data measure
$$
d\mu(x) = \frac{1}{n} \sum_{i=1}^n \delta_{x_i}(x),
$$
where $\delta_{x_i}$ is a Dirac distribution centered at $x_i$. While the data
measure space $(X,d\mu)$ does not satisfy the assumptions specified in Section
\ref{intro}, the construction of the operator $B_\varepsilon$ can be preformed
without issue; indeed, the construction only depends on the ability to compute
Euclidean distances between points and the ability to
integrate. Moreover, if $X$ consists of samples from a probability distribution
from a Riemannian manifold, then integrating with respect to the data measure
$d\mu$ can be viewed as Monte Carlo integration, which approximates integrals on
the manifold with error on the order of $1/\sqrt{n}$. Alternatively, the data
measure space $(X,d\mu)$ can be viewed as an intrinsic abstract measure space
and be considered independent from any manifold assumption. In either case,
the results regarding infinitesimal generators in Section \ref{intro} can guide
our choice of normalization. In the following we reiterate the construction of
$B_\varepsilon$ using matrix notation and describe computational considerations
which arise in this setting.  For notation, we use bold capital letters, e.g.,
$\mathbf{A}$, to denote matrices, lower case bold letters, e.g., $\mathbf{x}$,
to denote column vectors, and $\diag(\mathbf{x})$ to denote the diagonal matrix
with the vector $\mathbf{x}$ along the diagonal.

\subsection*{Step 1: Defining the kernel and weight function}
Let $\mathbf{K}_\varepsilon$ denote the $n \times n$ matrix with entries
$$
\mathbf{K}_{\varepsilon}(i,j) = h \left(\frac{|x_i-x_j|^2}{\varepsilon}
\right) \quad \text{for} \quad i,j = 1,\ldots,n,
$$
where $h$ is a smooth function of exponential decay, e.g., $h(|x|^2) = e^{-|x|^2}$. Let $\mathbf{w} \in
\mathbb{R}^n$ denote a weight function and define $ \mathbf{W} = \diag
\left( \mathbf{w} \right)$. For example, if an explicit $\mathbf{w}$ is not
given, then for a fixed parameter $-1 \le \beta \le 1$ we could define
$$
\mathbf{w}(i) = \left(\sum_{k=1}^n h \left( \frac{|x_i - x_k|}{\varepsilon}
\right) \right)^\beta,
\quad \text{for} \quad i=1,\ldots,n.
$$
\subsection*{Step 2: Sinkhorn iteration}
Given the weight matrix $\mathbf{W}$, we seek a diagonal matrix
$\mathbf{D}$ such that
$$
\mathbf{D}^{-1} \mathbf{K}_\varepsilon \mathbf{D}^{-1} \mathbf{W}^{-1}
\mathbf{1} = \mathbf{1}.
$$
The diagonal matrix $\mathbf{D}$ will be determined via Sinkhorn iteration
\cite{Sinkhorn1964}; that is to say, by alternating between normalizing the rows
and columns of the given matrix. In the case of a symmetric matrix, Sinkhorn
iteration can be phrased as the following iterative procedure: first, we set
$\mathbf{D}_0 = \mathbf{I}$, where $\mathbf{I}$ denotes the $n \times n$
identity matrix, and then define $\mathbf{D}_k$ inductively by
$$
\mathbf{D}_{k+1} = \diag \left(\mathbf{K} \mathbf{D}_k^{-1}
\mathbf{W}^{-1} \mathbf{1} \right).
$$
We claim that if we consider the limit
$$ 
\mathbf{D} := \lim_{k \rightarrow \infty} \mathbf{D_{k+1}}^{1/2}
\mathbf{D}_k^{1/2},
$$
then the matrix $\mathbf{D} \mathbf{K}_\varepsilon \mathbf{D}$ will be
bi-stochastic with respect to $\mathbf{W}^{-1}$ as desired. Indeed, the reason
for considering the limit of the geometric mean $\mathbf{D}_k^{1/2}
\mathbf{D}_{k+1}^{1/2}$ is that \cite{Sinkhorn1964} only guarantees that $$
\mathbf{D}_{2k} \rightarrow \alpha \mathbf{D}
\quad
\text{and}
\quad
\mathbf{D}_{2k+1} \rightarrow \frac{1}{\alpha} \mathbf{D}
$$
for some constant $\alpha > 0$. 

\subsection*{Step 2.5: Heuristic tricks for Sinkhorn iteration} We note that
the following trick seems to drastically improve the rate of convergence for
some Gaussian kernel matrices. Let $\mathbf{D}_0 = \mathbf{I}$,
$\mathbf{D}_0^\prime =
\mathbf{I}$, and define
$$
\mathbf{D}_{k+1} = \diag \left( \mathbf{K}_\varepsilon \mathbf{D}_k^{-1}
(\mathbf{D}_k^\prime)^{-1/2}
\mathbf{W}^{-1} \mathbf{1} \right)
\quad \text{and} \quad
\mathbf{D}_{k+1}^\prime = \diag \left( \mathbf{D}^{-1}_k \mathbf{K}_\varepsilon
\mathbf{D}^{-1}_k \mathbf{W}^{-1} \mathbf{1} \right).
$$
This iteration corresponds to first ``dividing'' by the row sums on both sides,
and then ``dividing'' by the square root of the row sums of the resulting matrix
on both sides of the resulting matrix; this iteration is motivated by a desire
to combat the imbalance between left and right iterates.

\subsection*{Step 3: Diagonalization}
In the discrete setting, the integral operator $B_\varepsilon$ can be expressed
$$
B_\varepsilon \mathbf{f} = \mathbf{D}^{-1} \mathbf{K}_\varepsilon
\mathbf{D}^{-1} \mathbf{W}^{-1} \mathbf{f},
$$
where $\mathbf{f}$ is a column vector.  Suppose that we compute the symmetric
eigendecomposition
$$
\mathbf{W}^{-1/2} \mathbf{D}^{-1} \mathbf{K}_\varepsilon \mathbf{D}^{-1}
\mathbf{W}^{-1/2} = \mathbf{U} \mathbf{S} \mathbf{U}^\top,
$$
where $\mathbf{U}$ is an orthogonal matrix of eigenvectors, and $\mathbf{S}$ is
a diagonal matrix of eigenvalues. Define
$$
\boldsymbol{\Phi} = \mathbf{W}^{1/2} \mathbf{U}.
$$
Suppose that $\boldsymbol{\varphi}_k$ denotes the $k$-th column of the matrix
$\boldsymbol{\Phi}$, and $\lambda_k$ denotes the $k$-th diagonal element of the
diagonal matrix $\mathbf{S}$. Then
$$
B_\varphi \boldsymbol{\varphi}_k = \lambda_k \boldsymbol{\varphi}_k.
$$
That is to say, $\boldsymbol{\varphi}_k$ is an eigenvector of $B_\varepsilon$ of
eigenvalue $\lambda_k$.
\subsection*{Step 4: Gradient estimation}
Suppose $h(|x|^2) = e^{-|x|^2}$. Let $\mathbf{X}$ denote the $n \times d$ data
matrix corresponding to $n$ points in $\mathbb{R}^d$. Motivated by Proposition
\ref{grad1} we define
$$
\boldsymbol{\nabla} \boldsymbol{\varphi}_k :=
\frac{ \mathbf{D}^{-1} \mathbf{K}_\varepsilon \mathbf{D}^{-1} \mathbf{W}^{-1}
\diag(\boldsymbol{\varphi}_k) \mathbf{X} - \lambda_k
\diag(\boldsymbol{\varphi}_k) \bar{\mathbf{X}}}{ \varepsilon \lambda_k},
$$
where 
$$
\Bar{\mathbf{X}} := \mathbf{D}^{-1} \mathbf{K}_\varepsilon \mathbf{D}^{-1}
\mathbf{W}^{-1} \mathbf{X}.  
$$
We note that the entire described construction can be completely implemented
with matrix operations and one call to a symmetric eigendecomposition routine.

\subsection{Numerical example: single measure eigenfunction gradient}
We illustrate Proposition \ref{grad1} with a numerical example. We sample $n =
1000$ points $\{x_i\}_{i=1}^n$ from a $3/2 \times 1$ rectangle with each
coordinate of each point sampled uniformly and independently from their
respective ranges.  We construct the bi-stochastic kernel $B_\varepsilon$ using
the discrete construction described in Section \ref{discrete} with the Gaussian
kernel $h(|x|^2) = e^{-|x|^2}$ and $w(x) = q_\varepsilon(x)$ (the kernel density
estimate). We compute the gradients $\boldsymbol{\nabla}
\boldsymbol{\varphi}_1$, $\boldsymbol{\nabla} \boldsymbol{\varphi}_2$,
$\boldsymbol{\nabla} \boldsymbol{\varphi}_3$, and $\boldsymbol{\nabla}
\boldsymbol{\varphi}_4$ using Proposition \ref{grad1} (Step 4 of the discrete
construction), and we plot the resulting vector fields in
Figure \ref{fig1}.  If $(x_1,x_2) \in \mathbb{R}^2$, then the first four
(nontrivial) Neumann Laplacian eigenfunctions on the $3/2 \times 1$ rectangle
are $\cos(\pi 2 x_1/ 3)$, $\cos(\pi x_2)$, $\cos(2 \pi 2 x_1/3)$, and $\cos(\pi
2 x_1/3) \cos( \pi
x_2)$; hence, the vector fields of the gradients of these functions plotted in
Figure \ref{fig1} appear as expected up to errors resulting from the discrete
approximation.

\begin{figure}[h!]
\centering
\begin{tabular}{cc}
\includegraphics[width=0.4\textwidth]{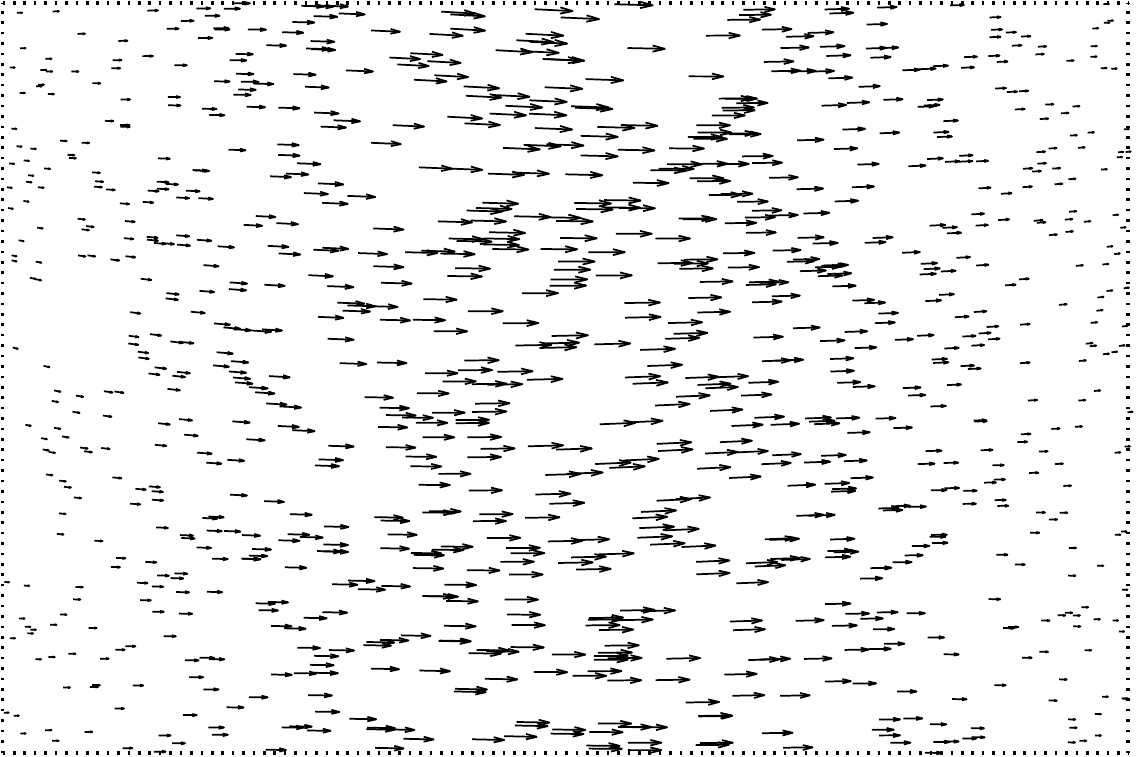} & 
\includegraphics[width=0.4\textwidth]{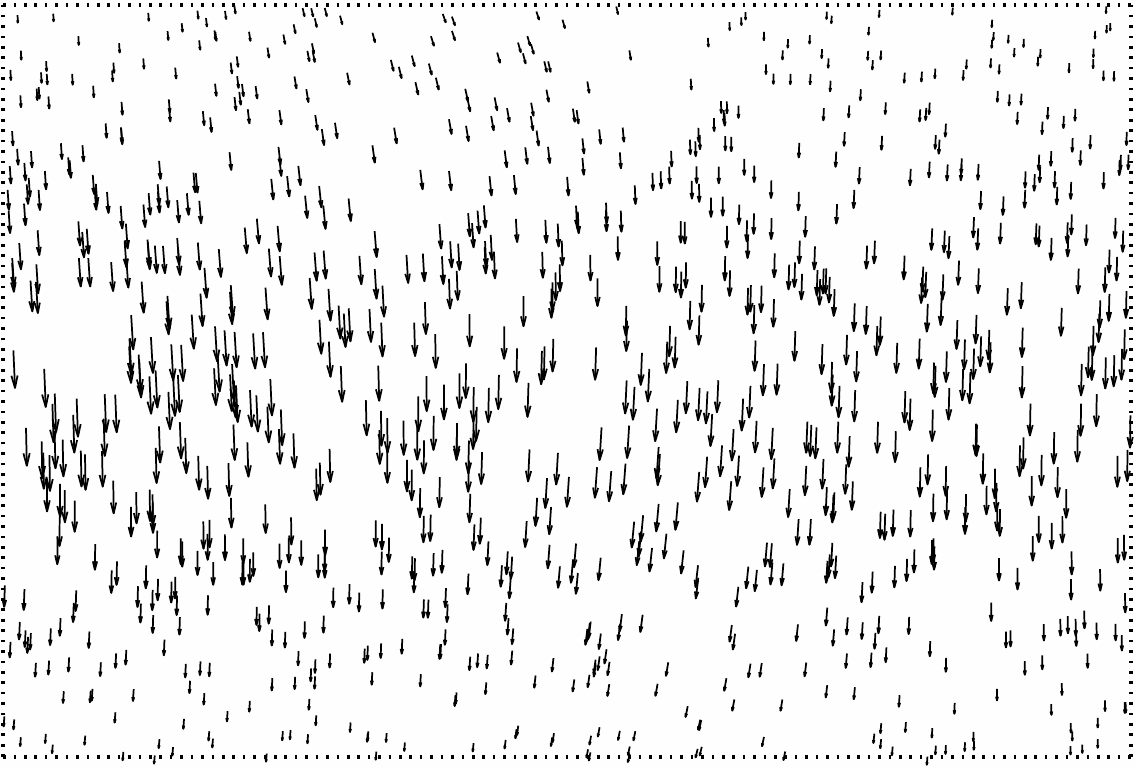} \\
\includegraphics[width=0.4\textwidth]{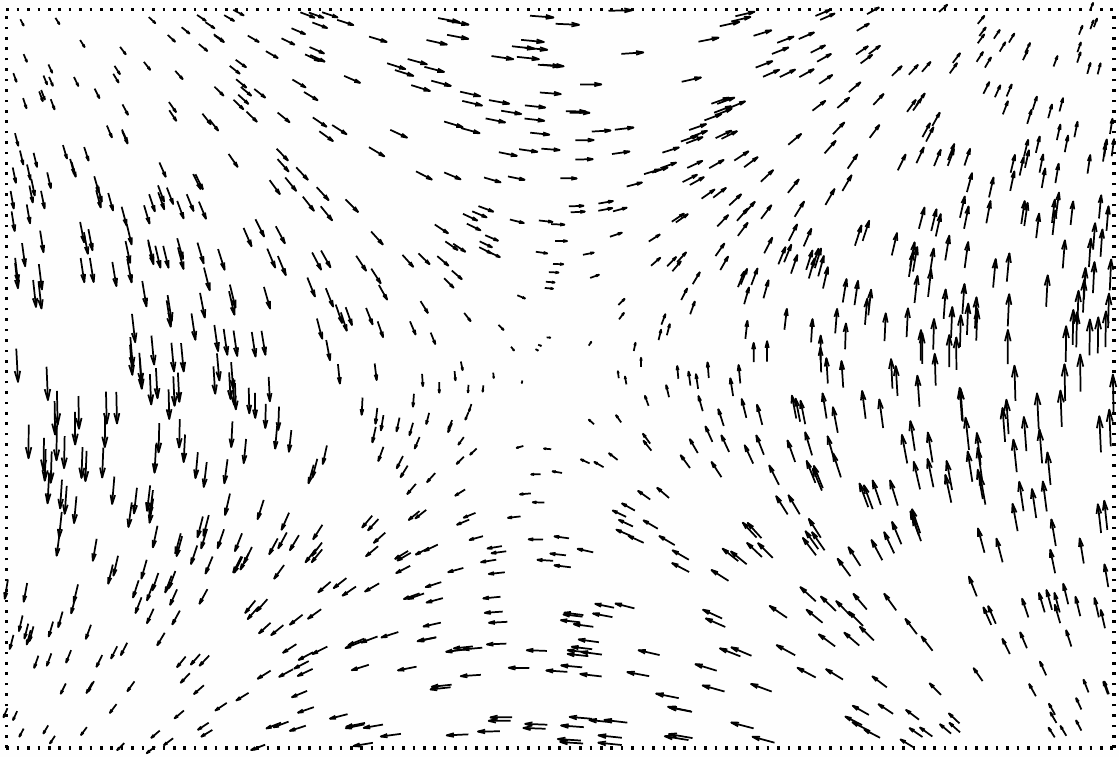} & 
\includegraphics[width=0.4\textwidth]{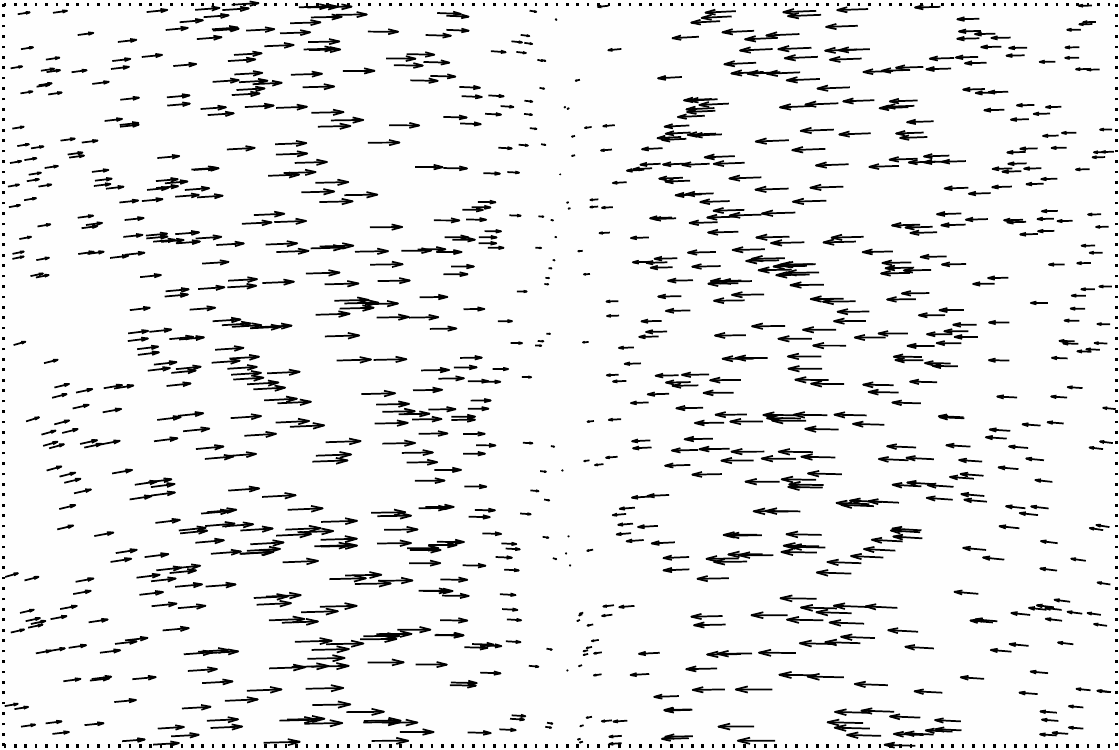} \\
\end{tabular}
\caption{The discrete vector fields $\boldsymbol{\nabla} \boldsymbol{\varphi}_1$
(top left), $\boldsymbol{\nabla} \boldsymbol{\varphi}_2$ (top right),
$\boldsymbol{\nabla} \boldsymbol{\varphi}_3$ (bottom left), and
$\boldsymbol{\nabla} \boldsymbol{\varphi}_4$ (bottom right) for our numerical
example.} \label{fig1}
\end{figure}
\subsection{Reference Measure} \label{discreteref}
Next, we describe the discrete formulation of the operator
$C_\varepsilon$.  Suppose that finite sets $\{x_i\}_{i=1}^n \subset \mathbb{R}^d$ and $\{r_j\}_{j=1}^m \subset
\mathbb{R}^d$ are given. These sets induce measures
$$
d\mu(x) = \frac{1}{n} \sum_{i=1}^n \delta_{x_i}(x) \quad \text{and} \quad 
d\nu(r) = \frac{1}{m} \sum_{j=1}^m \delta_{r_j}(r).
$$
In the following, we reiterate the construction of the operator
$C_\varepsilon$ in matrix notation. Moreover, we describe computational
considerations under the assumption that the reference set $\{r_j\}_{j=1}^m$ has
significantly less elements than the data set $\{x_i\}_{i=1}^n$. 

\subsection*{Step 1: Defining the kernel and weight functions}
We begin by defining the $n \times m$ kernel matrix
$$
\mathbf{K}_\varepsilon(i,j) = h\left( \frac{|x_i-r_j|}{\varepsilon} \right),
\quad \text{for} \quad i=1,\ldots,n \quad \text{and} \quad j=1,\ldots,m,
$$
where $h$ is a smooth function of exponential decay, and where $|\cdot|$ denotes
the Euclidean norm in $\mathbb{R}^d$. Suppose that $\mathbf{w}
\in \mathbb{R}^n$ and $\mathbf{v} \in \mathbb{R}^m$ are weight vectors and set
$\mathbf{W} = \diag(\mathbf{w})$ and $\mathbf{V} = \diag(\mathbf{v})$.  For
example, if $\mathbf{w}$, and $\mathbf{v}$ are not explicitly given, then for
$-1 \le \beta \le 1$ and $-1 \le \gamma \le 1$ we could define 
$$
\mathbf{v}(i) = \left(\sum_{k=1}^n h \left(\frac{|x_i - x_k|}{\varepsilon}
\right) \right)^\beta
\quad \text{and} \quad
\mathbf{v}(j) = \left(\sum_{k=1}^m h \left(\frac{|r_j - r_k|}{\varepsilon}
\right) \right)^\gamma,
$$
for $i=1,\ldots,n$ and $j=1,\ldots,m$, respectively.

\subsection*{Step 2: Sinkhorn Iteration} Next, Sinkhorn iteration is used to compute a diagonal
matrix $\mathbf{D}$ such that
$$
\mathbf{D}^{-1} \mathbf{K}_\varepsilon \mathbf{V}^{-1}
\mathbf{K}_\varepsilon^\top
\mathbf{D}^{-1} \mathbf{W}^{-1} \mathbf{1} = \mathbf{1}
$$
where $\mathbf{1}$ denotes a vector of ones. The same Sinkhorn Iteration
described in the previous section can be used. We remark that in this case it
may be more efficient to applying each matrix to the vector successively in the
product rather than computing the matrix product first.

\subsection*{Step 3: Eigendecomposition} We want to compute the
eigenvectors of the operator
$$
C_\varepsilon \mathbf{f} = \mathbf{D}^{-1} \mathbf{K}_\varepsilon
\mathbf{V}^{-1} \mathbf{K}_\varepsilon^\top \mathbf{D}^{-1} \mathbf{W}^{-1}
\mathbf{f}.
$$
Suppose that we compute the singular value decomposition
$$
\mathbf{W}^{-1/2} \mathbf{D}^{-1} \mathbf{K}_\varepsilon
\mathbf{V}^{-1/2} 
=
\mathbf{U} \mathbf{S} \mathbf{V}^\top,
$$
where $\mathbf{U}$ is an orthogonal matrix of left singular vectors,
$\mathbf{S}$ is a diagonal matrix of singular values, and $\mathbf{V}$ is a
orthogonal matrix of right singular vectors. Define
$$
\boldsymbol{\Phi} := \mathbf{W}^{1/2} \mathbf{U}.
$$
By the definition of the singular value decomposition, it immediately follows
that if $\boldsymbol{\varphi}_k$ is the $k$-th column of $\boldsymbol{\Phi}$, and
$\lambda_k$ is the $k$-th diagonal element of $\mathbf{S}^2$, then
$$
C_\varepsilon \boldsymbol{\varphi}_k = \lambda_k \boldsymbol{\varphi}_k,
$$
that is, $\boldsymbol{\varphi}_k$ is an eigenvector of $C_\varepsilon$ of
eigenvalue $\lambda_k$.

\subsection*{Step 4: Eigenfunction Gradients}
Let $\mathbf{R}$ denote the $m \times d$ matrix encoding the reference set
$\{r_j\}_{j=1}^m \subset \mathbb{R}^d$. Motivated by Proposition \ref{grad2} we
define
$$
\boldsymbol{\nabla} \boldsymbol{\varphi}_j := \frac{
 (F_\varepsilon \boldsymbol{\varphi}_j)  \mathbf{R}
-  \lambda_j \diag \left( \boldsymbol{\varphi}_j \right) \bar{\mathbf{R}}_\mathbf{X}}{\lambda_j \varepsilon},
$$
where
$$
F_\varepsilon \mathbf{f} := \mathbf{D}^{-1} \mathbf{K}_\varepsilon
 \diag \left(\mathbf{K}^\top_\varepsilon \mathbf{D}^{-1} \mathbf{W}^{-1}
\mathbf{f} \right) \mathbf{V}^{-1}
\quad \text{and}
\quad
\bar{\mathbf{R}}_\mathbf{X} = (F_\varepsilon \mathbf{1}) \mathbf{R}.
$$
Here $\mathbf{f}$ denotes an arbitrary $n \times 1$ vector, while $\mathbf{1}$
denotes either an $n \times 1$ or an $m \times 1$ vector of ones depending on
the context.

\subsection{Numerical example: reference measure eigenfunction gradient}
First, we generate a dataset of $n = 1000$ points $\{x_i\}_{i=1}^n$ sampled
uniformly at random from the unit disc. Second, we select a reference set of $m
= 100$ points using a pivoted Gram-Schmidt procedure on the kernel matrix 
$$ 
\mathbf{K}(i,j) = e^{-|x_i - x_j|^2/\varepsilon}.  
$$ 
\begin{figure}[h!]
\centering
\begin{tabular}{cc}
\includegraphics[width=0.3\textwidth]{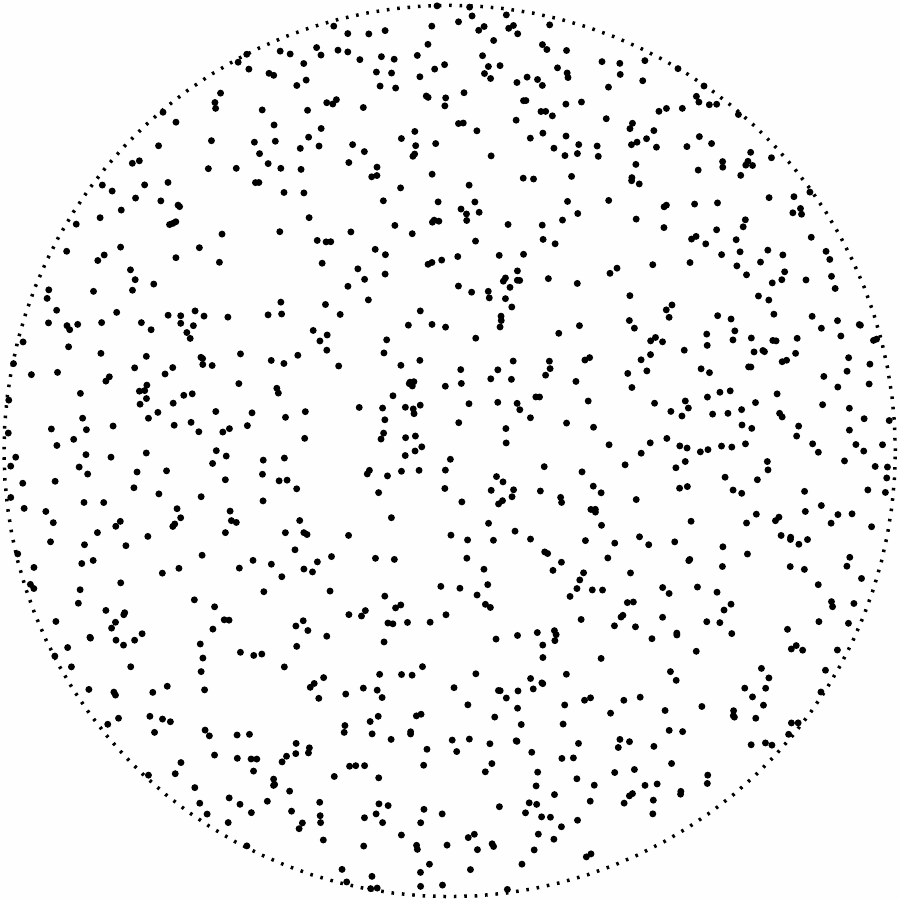} &
\includegraphics[width=0.3\textwidth]{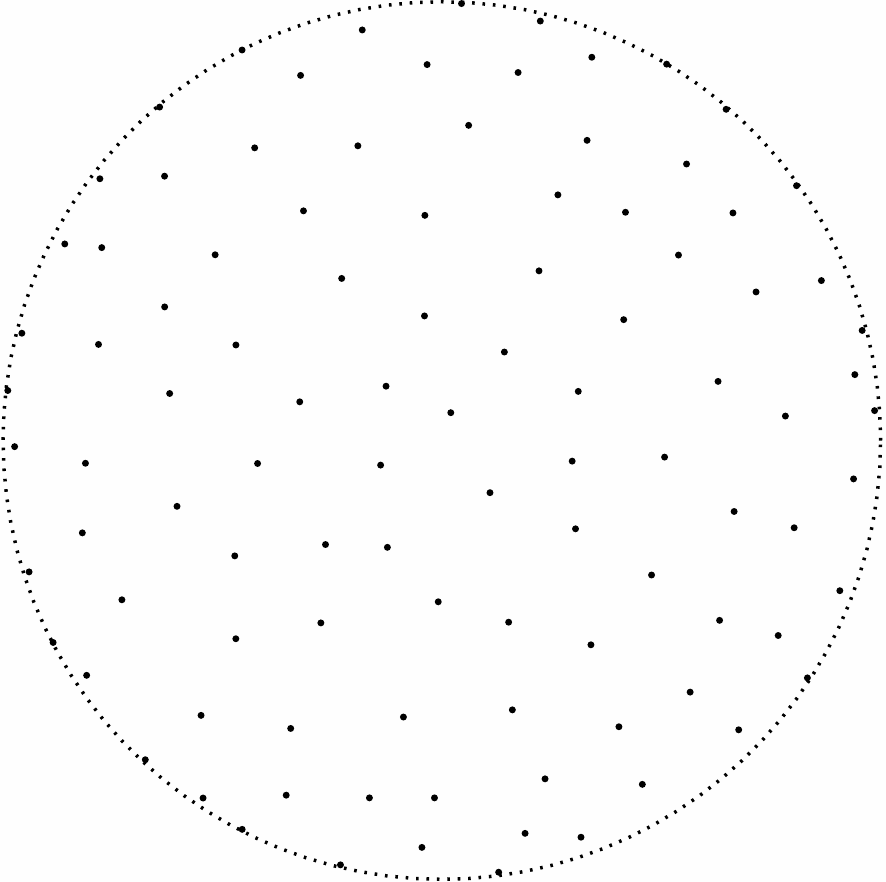}
\end{tabular}
\caption{We plot the data set $\{x_i\}_{i=1}^n$ (left) and the reference set
$\{r_j\}_{j=1}^m$ (right).} \label{fig10}
\end{figure}

\noindent Specifically, we choose the indices  $i_1,\ldots i_m$ of the reference
points from the data set $\{x_i\}_{i=1}^n$ inductively as follows. Let
$\{\mathbf{k}_i^{(1)}\}_{i=1}^n$ denote the columns of $\mathbf{K}$ and define
$i_1 = \argmax_{i} | \mathbf{k}^{(1)}_i |$ where $|\cdot|$ denote the Euclidean
norm in $\mathbb{R}^n$. Given $i_1,\ldots,i_j$ we choose $i_{j+1}$ by
$$
i_{j+1} = \argmax_i |\mathbf{k}_i^{(j+1)}| \quad \text{where} \quad
\mathbf{k}_i^{(j+1)} = \mathbf{k}_i^{(j)} - \frac{\mathbf{k}_i^{(j)} \cdot
\mathbf{k}_{i_j}}{ |\mathbf{k}_{i_j}^{(j)} |^2} \mathbf{k}_{i_j}^{(j)} \quad
\text{for} \quad i =1,\ldots,n,
$$
where $\cdot$ denotes the Euclidean inner product in $\mathbb{R}^n$.  That is to
say, at each step we choose the column of $\mathbf{K}$ with the largest $L^2$
norm, and then orthogonalize all of the columns to the selected column and
proceed iteratively.  Given the data set $\{x_i\}_{i=1}^n$ and the reference set
$\{r_j\}_{j=1}^m$ we construct the operator $C_\varepsilon$ and its
eigenfunctions as described in
Section \ref{discrete}.  
Using Proposition \ref{grad2} (whose discrete
formulation is described in Step 4 in Section \ref{discreteref}) we compute an
approximation of the gradient of the first two (nontrivial) eigenfunctions
$\boldsymbol{\nabla}\boldsymbol{\varphi}_1$, and $\boldsymbol{\nabla}
\boldsymbol{\varphi}_2$  which are plotted in Figure \ref{fig20}. The first
(nontrivial) Neumann Laplacian eigenvalue on the disc is of multiplicity two,
and the associated eigenspace is spanned by the functions 
$$
J_1( \lambda_1 r) \cos ( \theta ) 
\quad \text{and} \quad 
J_1( \lambda_1 r) \sin ( \theta ) 
\quad \text{where} \quad \lambda_1 \approx 1.84
$$
is the smallest positive root of the equation $J^\prime_1(x) = 0$; 
therefore, the discrete approximations of the gradients $\boldsymbol{\nabla}
\boldsymbol{\varphi}_1$ and $\boldsymbol{\nabla} \boldsymbol{\varphi}_2$ in Figure
\ref{fig20} appear as expected.

\begin{figure}[h!]
\centering
\begin{tabular}{cccc}
\includegraphics[width=0.3\textwidth]{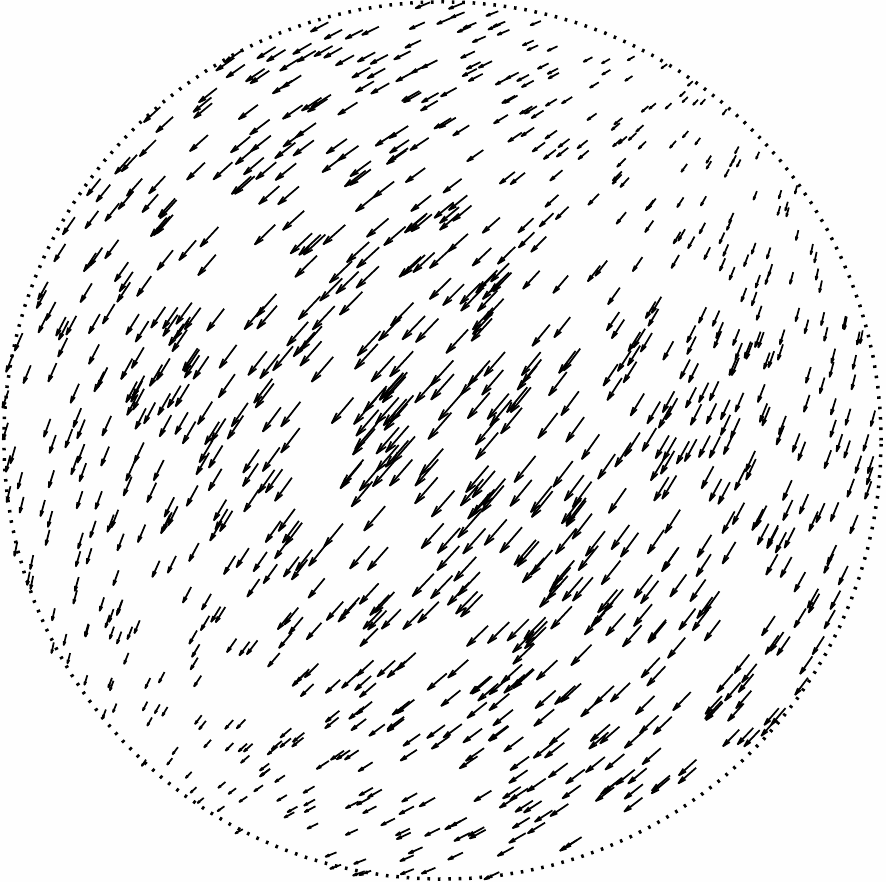} &
\includegraphics[width=0.3\textwidth]{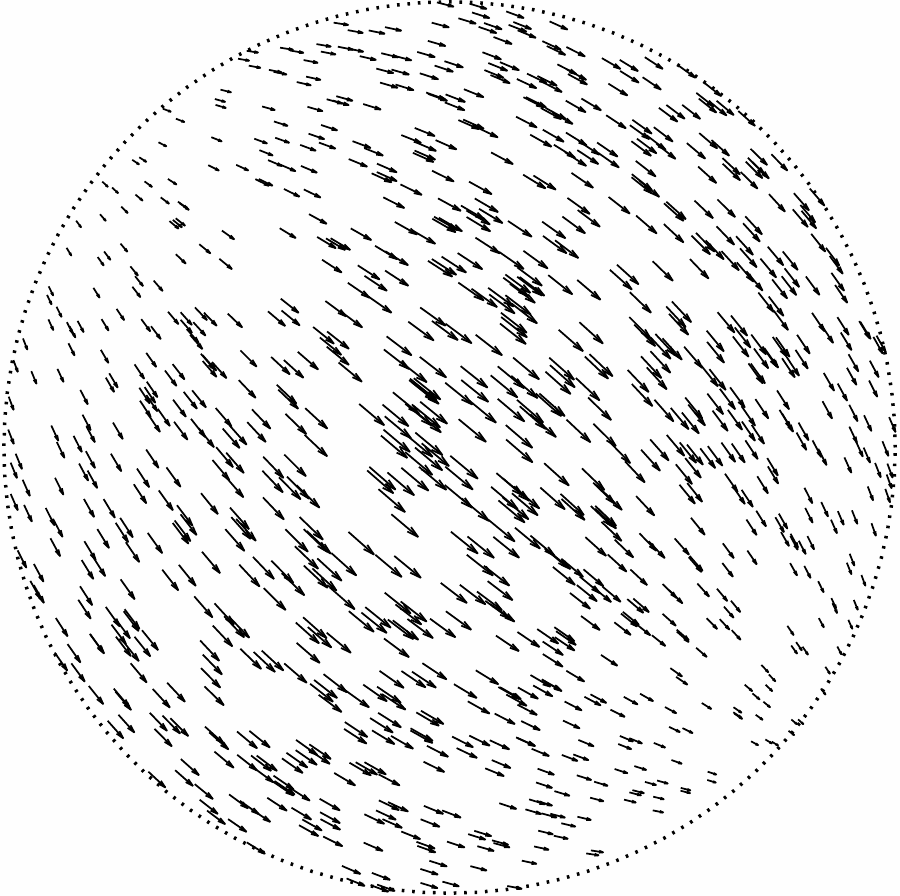} &
\end{tabular}
\caption{We plot the discrete vector fields
$\boldsymbol{\nabla} \boldsymbol{\varphi}_1$ (left), and $\boldsymbol{\nabla}
\boldsymbol{\varphi}_2$ (right)}.
\label{fig20}
\end{figure}

\section{Proof of Theorem \ref{thm1}} \label{proofthm1}

\subsection{Notation}  \label{notation1}
Recall that $X$ is a compact smooth Riemannian manifold with a measure $d\mu(x)
= q(x) dx$ where $q \in C^3(X)$ is a positive density function with respect to
the volume measure $dx$ on $X$. Furthermore, recall that $X$ is
isometrically embedded in $\mathbb{R}^d$, and that the kernel
$k_\varepsilon(x,y)$ is defined 
$$
k_\varepsilon(x,y) = h \left( \frac{|x-y|^2}{\varepsilon} \right),
$$ 
where $h$ is a smooth function of exponential decay, and where $|\cdot|$ denotes
the Euclidean norm on $\mathbb{R}^d$. Finally, recall that $w \in
C^3(X)$ is a positive weight function which defines the measure $d\hat{\mu}(x)
:= d\mu(x)/w(x)$. 

\subsection{Existence of bi-stochastic normalization} \label{existence}
First, we argue why there exists a normalization function $d \in
C^\infty(X)$ such that the kernel $b_\varepsilon : X \times X \rightarrow
\mathbb{R}$ defined by
$$
b_\varepsilon(x,y) = \frac{k_\varepsilon(x,y)}{d(x) d(y)}
\quad \text{is bi-stochastic with respect to} \quad d\hat{\mu}.
$$
By a result of Knopp and Sinkhorn \cite{KnoppSinkhorn1968} a positive
normalization function exists (also see Theorem 5.1 in
\cite{BorweinLewisNussbaum1994}).  However, if such a function $d(x)$ exists,
then it satisfies the equation 
$$
d(x) = \int_X \frac{k_\varepsilon(x,y)}{d(y)} d\hat{\mu}(y).
$$
Since the kernel $k_\varepsilon$ is smooth, we conclude that $d$ is also smooth.

\subsection{Useful Lemmas}
Our proof strategy is motivated by Lafon \cite{Lafon2004}, and we adopt similar
notation. Let
$$
m_0 = \int_{\mathbb{R}^d} h(|t|^2) dt,
\quad
\text{and}
\quad
m_2 = \int_{\mathbb{R}^d} t_1^2 h(|t|^2) dt,
$$
where $dt$ denotes Lebesgue measure on $\mathbb{R}^d$. Consider the integral
operator $G_\varepsilon : L^2(X,dx) \rightarrow L^2(X,dx)$ by
$$
(G_\varepsilon f)(x) = \frac{1}{\varepsilon^{d/2}} \int_X
k_\varepsilon(x,y) f(y) dy.
$$
The following Lemma is due to \cite{CoifmanLafon2006}, also see
\cite{Singer2006}.

\begin{lemma} \label{lem1} Suppose $0 < \gamma < 1/2$ and $f \in C^3(X)$. If $x
\in X$ is at least distance $\varepsilon^\gamma$ from $\partial X$, then
$$
(G_\varepsilon f) (x) = m_0 f(x) + \frac{\varepsilon m_2}{2} \left(E(x)f(x) -
\Delta f(x) \right) + \mathcal{O} \left( \varepsilon^{2} \right),
$$
where $E(x)$ is a scalar function of the curvature at $x$.
\end{lemma}

\noindent Additionally, we need an expansion for $G_\varepsilon f$ near the boundary. The
following result is due to \cite{Lafon2004}.

\begin{lemma} \label{lem2}
Suppose $0 < \gamma < 1/2$ and $f \in C^3(X)$. Then uniformly over $x \in X$
within distance $\varepsilon^{\gamma}$ of $\partial X$
$$
G_\varepsilon f(x) = m_{0,\varepsilon}(x) f(x_0) + \sqrt{\varepsilon}
m_{1,\varepsilon}(x) \frac{\partial f}{\partial n} (x_0) +
\mathcal{O}(\varepsilon),
$$
where $x_0$ is the closest point in $\partial X$ to $x$ (with respect to
Euclidean distance), and $m_{0,\varepsilon}(x)$ and $m_{1,\varepsilon}(x)$ are
bounded functions of $x$ and $\varepsilon$.
\end{lemma}

\subsection{Intermediate results}

Before proving Theorem \ref{thm1} we establish two results: Proposition
\ref{prop1} and Proposition \ref{prop2}, which provide local expansions for the
bi-stochastic operator $B_{\varepsilon}$ away from the boundary, and near the
boundary, respectively.

\begin{proposition} \label{prop1}
Suppose $0 < \gamma < 1/2$ and $f \in C^3$. If $x \in X$ is at least distance
$\varepsilon^\gamma$ from $\partial X$, then
$$
B_{\varepsilon} f(x) = f(x)- \frac{ \varepsilon m_2}{2 m_0} \left( \frac{ \Delta
\left( f(x) \left( \frac{q(x)}{w(x)} \right)^{1/2} \right) }{ \left(
\frac{q(x)}{w(x)}\right)^{1/2} } - \frac{ \Delta \left(
\frac{q(x)}{w(x)}\right)^{1/2} }{ \left( \frac{q(x)}{w(x)}\right)^{1/2} } f(x)
\right) + \mathcal{O}\left( \varepsilon^{2} \right).
$$
\end{proposition}

\begin{proof}
Since $b_{\varepsilon}$ is bi-stochastic with respect to $d\hat{\mu}$, we have
the following integral equation
$$
1 = \int_X b_\varepsilon(x,y) d\hat{\mu}(y) = \int_X
\frac{k_\varepsilon(x,y)}{d(x) d(y)} \frac{q(y)}{w(y)} dy.
$$
Expanding this integral equation using Lemma \ref{lem1} gives
$$
1 = \varepsilon^{d/2} \frac{1}{d} \left( m_0 \frac{1}{d} \cdot
\frac{q}{w} + \frac{\varepsilon m_2}{2} \left( E \cdot \frac{1}{d}
\cdot \frac{q}{w} - \Delta \left( \frac{1}{d} \cdot
\frac{q}{w} \right) \right) + \mathcal{O}\left(\varepsilon^2 \right)
\right),
$$
where the argument $x$ of the functions $d(x)$, $q(x)$,
$w(x)$, and $E(x)$ have been suppressed for notational brevity.
By rearranging terms we have
$$
d^2 = \varepsilon^{d/2} m_0
\frac{q}{w} \left( 1 + \frac{\varepsilon m_2}{2 m_0}
\left( E - \frac{\Delta\left( q /(d\cdot w) \right)
}{ q/(d \cdot w)} \right) +
\mathcal{O}(\varepsilon^2) \right).
$$
Next, we compute an expansion for $B_{\varepsilon}$ applied to an
arbitrary function $f \in C^3(X)$
$$
B_{\varepsilon} f(x) = \int_X \frac{k_\varepsilon(x,y)}{d(x) d(y)} f(y)
\frac{q(y)}{w(y)} dy.
$$
Expanding the right hand side using Lemma \ref{lem1} gives
$$
B_{\varepsilon} f = \varepsilon^{d/2} \frac{1}{d} \left( 
m_0 \frac{1}{d} \cdot f \cdot \frac{q}{w} 
+ \frac{\varepsilon
m_2}{2} \left( E \cdot \frac{1}{d} \cdot f \cdot \frac{q}{w} - 
\Delta \left( \frac{1}{d} \cdot f \cdot \frac{q}{w} \right)
\right) + \mathcal{O} \left( \varepsilon^2 \right) \right).
$$
Rearranging terms gives
$$
B_{\varepsilon} f = m_0 \varepsilon^{d/2} f \cdot \frac{1}{d^2} \cdot
\frac{q}{w} \left( 1 + \frac{\varepsilon m_2}{2 m_0}
\left( E + \frac{\Delta \left(f \cdot q/(d \cdot w)
\right)}{f \cdot q/(d \cdot w)} \right) 
+ \mathcal{O} \left( \varepsilon^2 \right) \right).
$$
Substituting in the expansion for $d^2$ into this equation yields
$$
B_{\varepsilon} f = f \left( 1 + \frac{\varepsilon m_2}{2
m_0} \left( E - \frac{\Delta \left( f \cdot
q/(d \cdot w) \right)}{f \cdot
q/(d \cdot w)} \right) \right) \left(
1 + \frac{\varepsilon m_2}{2 m_0}
\left( E - \frac{\Delta \left( q /(d \cdot w)
\right)}{ q /(d \cdot w)} \right) \right)^{-1}
+ \mathcal{O} \left( \varepsilon^2 \right).
$$
Expanding the inverse term in a Taylor expansion gives
$$
B_{\varepsilon} f = f - \frac{\varepsilon m_2}{2 m_0} \left(
\frac{\Delta \left( f \cdot q/(d \cdot w)
\right)}{q/(d \cdot w)} 
- \frac{\Delta \left( q /(d \cdot w)
\right)}{ q /(d \cdot w)} f
\right) + \mathcal{O}\left(\varepsilon^2 \right).
$$
The expression for $d^2$ implies that
$$
d = m_0^{1/2} \varepsilon^{d/4} \frac{q^{1/2}}{w^{1/2}} \left(1 +
\mathcal{O}\left( \varepsilon \right) \right).
$$
Substituting this representation for $d$ into the expansion for $B_\varepsilon
f$ yields
$$
B_{\varepsilon} f = f - \frac{\varepsilon m_2}{2 m_0} \left( \frac{\Delta \left(
f \cdot \left(\frac{q}{w}\right)^{1/2} \right)} { \left(\frac{q}{w}\right)^{1/2}
} - \frac{\Delta  \left( \frac{q}{w} \right)^{1/2} }{ \left( \frac{q}{w}
\right)^{1/2} } f \right) + \mathcal{O}\left(\varepsilon^2 \right),
$$
as was to be shown.
\end{proof}

\begin{proposition} \label{prop2}
Let $k > 0$ be fixed.  Suppose $0 < \gamma < 1/2$ and $f \in E_k$.  Then
uniformly over $x \in X$ within distance $\varepsilon^\gamma$ of $\partial
X$ 
$$
B_{\varepsilon} f(x) = f(x_0) + \mathcal{O}(\varepsilon),
$$
where $x_0$ is the closest point on $\partial X$ to $x$ (with respect to
Euclidean distance).
\end{proposition}

\begin{proof}
Since $B_{\varepsilon}$ is bi-stochastic with respect to
$q(y)/w(y) dy$ we have the following integral equation
$$
1 = \int_X \frac{k_\varepsilon(x,y)}{d(x) d(y)}
\frac{q(y)}{w(y)} dy.
$$
Expanding the integral equation using Lemma \ref{lem2} gives
$$
1 = \varepsilon^{d/2} \frac{1}{d} \left( m_{0,\varepsilon}
\frac{q_{x_0}}{d_{x_0} \cdot w_{x_0}} + \varepsilon^{1/2} m_{1,\varepsilon}
\partial_n \left(  \frac{q_{x_0}}{d_{x_0} \cdot w_{x_0}} \right) +
\mathcal{O} \left( \varepsilon \right) \right), 
$$
where $\partial_n$ denotes the normal derivative. Note, that for notional
brevity the argument $x$ of the functions $d(x)$, $m_{0,\varepsilon}(x)$, and
$m_{1,\varepsilon}(x)$ has been suppressed, and we write $d_{x_0} = d(x_0)$,
$q_{x_0} = q(x_0)$, and $w_{x_0} = w(x_0)$. Rearranging terms gives
$$
d \cdot d_{x_0} = \varepsilon^{d/2} \frac{q_{x_0}}{w_{x_0}} \left(
m_{0,\varepsilon} + \varepsilon^{1/2} m_{1,\varepsilon}  \frac{\partial_n
\left( q_{x_0} /(d_{x_0} \cdot w_{x_0} \right)}{q_{x_0}/(d_{x_0} \cdot w_{x_0})}
+ \mathcal{O} \left(\varepsilon \right) \right).
$$
Next we compute the expansion for the operator $B_{\varepsilon}$ applied to a
function $f \in E_k$
$$
B_{\varepsilon} f(x) = \int_X \frac{k_\varepsilon(x,y)}{d(x) d(y)}
f(y) \frac{q(y)}{w(y)} dy.
$$
Applying Lemma \ref{lem2} to the right hand side yields
$$
B_{\varepsilon} f = \varepsilon^{d/2} \frac{1}{d \cdot d_{x_0}}
\frac{q_{x_0}}{w_{x_0}} \left( m_{0,\varepsilon}  f_{x_0} +
\varepsilon^{1/2} m_{1,\varepsilon}  \frac{\partial_n \left( 
f_{x_0} \cdot q_{x_0} /\left( d_{x_0} \cdot w_{x_0}\right) \right)}{ f_{x_0}
\cdot q_{x_0}/\left(d_{x_0} \cdot w_{x_0} \right) } +
\mathcal{O} \left( \varepsilon \right) \right).
$$
Substituting in the expansion for $d \cdot d_{x_0}$ into this equation gives
$$
B_{\varepsilon} f =  f_{x_0} \left( 1 + 
\varepsilon^{1/2} \frac{m_{1,\varepsilon}}{m_{0,\varepsilon}} 
\frac{\partial_n \left( f_{x_0}  \cdot q_{x_0} / (d_{x_0} \cdot w_{x_0})
\right)}{f_{x_0} \cdot
 q_{x_0} /(d_{x_0} \cdot w_{x_0})} \right)\left(
1 + \varepsilon^{1/2}
\frac{m_{1,\varepsilon}}{m_{0,\varepsilon}}  \frac{\partial_n \left(
 q_{x_0} /(d_{x_0} \cdot  w_{x_0}) \right)}{
q_{x_0}/(d_{x_0} \cdot w_{x_0})} \right)^{-1} + \mathcal{O}(\varepsilon).
$$
Using the fact that $f \in E_k$ verifies the Neumann condition at $x_0$, we
can take $f$ out of any derivative across the boundary, and the expression
simplifies to
$$
B_{\varepsilon} f(x)  = f(x_0) + \mathcal{O}\left( \varepsilon \right),
$$
which completes the proof.
\end{proof}

\subsection{Proof of Theorem \ref{thm1}}
\begin{proof}[Proof of Theorem \ref{thm1}]
It remains to combine the results of Propositions \ref{prop1} and \ref{prop2}.
Define the set 
$$
X_\varepsilon := \{ x \in X : d(x,\partial X) > \varepsilon^\gamma \},
$$
where $d(x,\partial X) = \inf_{x_0 \in \partial X} |x -
x_0|$, and consider the regions $X_\varepsilon$ and $X \setminus
X_\varepsilon$ illustrated in Figure \ref{fig2}.

\begin{figure}[h!]
\centering
\includegraphics[width=.4\textwidth]{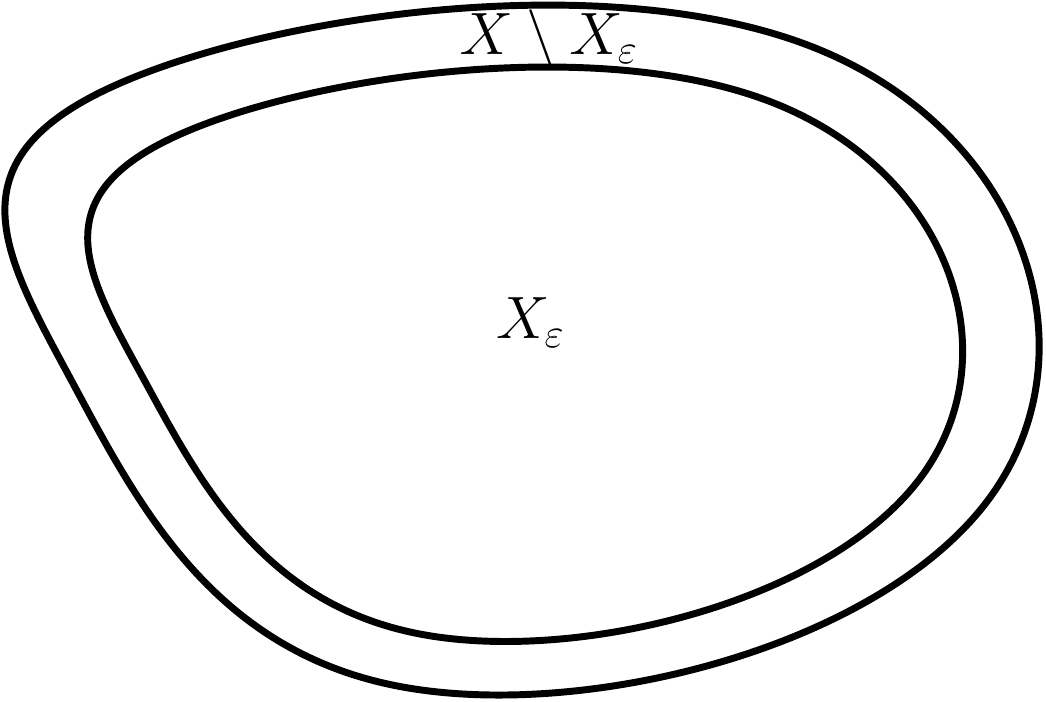}
\caption{We partition $X$ into $X_\varepsilon$ and $X \setminus X_\varepsilon$
where $X_\varepsilon = \{x \in X : d(x,\partial X) > \varepsilon^\gamma\}$.}
\label{fig2}
\end{figure}

\noindent Suppose $f \in E_k$ is given. First, if $x \in X_\varepsilon$, then we apply
Proposition \ref{prop1} to conclude that
$$
\frac{I - B_{\varepsilon}}{\varepsilon} f(x) = \frac{ m_2}{2
m_0} \left( \frac{ \Delta \left( f(x) \left( \frac{q(x)}{w(x)} \right)^{1/2}
\right)}{ \left( \frac{q(x)}{w(x)} \right)^{1/2}} - \frac{ \Delta \left(
\frac{q(x)}{w(x)} \right)^{1/2} }{ \left( \frac{q(x)}{w(x)} \right)^{1/2}} f(x)
\right) + \mathcal{O}\left( \varepsilon \right),
$$
that is to say, the desired result holds pointwise in the region
$X_\varepsilon$.  Second, by Proposition \ref{prop2} uniformly over $x \in X \setminus X_\varepsilon$
$$
\frac{I - B_{\varepsilon}}{\varepsilon} f(x) = \frac{f(x_0) - f(x)}{\varepsilon}
+ \mathcal{O}(1).
$$
Since $X$ is compact, we can use an $L^\infty$ bound to control the contribution
from the region $X \setminus X_\varepsilon$ to the $L^2$ norm; since the volume
of the region is order $\varepsilon^\gamma$, it follows by Proposition
\ref{prop2} that the contribution to the $L^2$ norm is at most order
$$
\varepsilon^\gamma \left(\frac{f(x_0) - f(x)}{\varepsilon} + 1\right) = 
\mathcal{O} \left( \varepsilon^{3 \gamma - 1} \right).
$$
Indeed, since $|x_0 - x| < \varepsilon^\gamma$, and $f$ satisfies the Neumann
condition at $x_0$ it follows that $f(x_0) = f(x) +
\mathcal{O}\left(\varepsilon^{2\gamma}\right)$, which implies that the entire
expression is at most order $\varepsilon^{3 \gamma -1}$. Choosing $\varepsilon =
0.49$ implies that the contribution of the region $X \setminus X_\varepsilon$ to
the $L^2$ norm goes to zero. Thus we conclude that 
$$
\frac{I - B_{\varepsilon}}{\varepsilon} f \rightarrow \frac{m_2}{2 m_0} \left(
\frac{\Delta \left( f \left(\frac{q}{w}\right)^{1/2} \right)}{\left( \frac{q}{w}
\right)^{1/2}}  - 
\frac{\Delta \left( \frac{q}{w} \right)^{1/2} }{\left(
\frac{q}{w}\right)^{1/2}}  f \right) \quad \text{in $L^2$ norm as} \quad
\varepsilon \rightarrow 0,
$$
as was to be shown.
\end{proof}

\subsection{Proof of Corollaries \ref{cor1}}
\begin{proof}[Proof of Corollary \ref{cor1}]
Recall that 
$$
q_\varepsilon(x) := \int_X k_\varepsilon(x,y) d\mu(y).
$$
If $y \in X$ is fixed, then $k_\varepsilon(\cdot,y) \in C^\infty(X)$ since
$k_\varepsilon(x,y) = h(-|x-y|^2/\varepsilon)$, where the function $h$ is smooth
by assumption. It follows that $q_\varepsilon \in C^\infty(X)$. Expanding
$q_\varepsilon$ via Lemma \ref{lem1} gives
$$
q_\varepsilon = \varepsilon^{d/2} m_0 q \left( 1 +
\frac{\varepsilon m_2}{2 m_0} \left( E - \frac{\Delta q}{q} \right) +
\mathcal{O} \left( \varepsilon^{2} \right) \right) = \varepsilon^{d/2} m_0
\left( q + \mathcal{O}(\varepsilon) \right),
$$
where the argument $x$ of the functions $q_\varepsilon(x)$, $E(x)$, and $q(x)$
has been suppressed for notational brevity. Substituting this expression for
$q_\varepsilon$ into the final step of the proof of Theorem \ref{thm1} yields Corollary
\ref{cor1}. 
\end{proof}

\subsection{Proof of Corollary \ref{cor2}}
\begin{proof}[Proof of Corollary \ref{cor2}]
The key observation in the proof is that the heat kernel $e^{-t \Delta}$ can be
expressed
$$
(1 - \varepsilon \Delta)^{t/\varepsilon} = e^{-t\Delta},
$$
see \cite{CoifmanLafon2006}.
Since the Neumann eigenfunctions of the Laplace-Beltrami operator are an
orthogonal basis for $L^2(X)$, any function $f \in L^2(X)$ can be written
$$
f = f_\varepsilon + g_\varepsilon
$$
where $f_\varepsilon \in E_k$ for some $k > 0$, and where $\|g_\varepsilon
\|_{L^2} \le \varepsilon$. Since $B_\varepsilon^{t/\varepsilon}$ and $e^{-t
\Delta}$ are bounded operators, by the triangle inequality
$$
\| B_\varepsilon^{t/\varepsilon} f - e^{-t \Delta} f\|_{L^2} \le
\|B_\varepsilon^{t/\varepsilon} f_\varepsilon  - e^{-t \Delta} f_\varepsilon\|_{L^2} +
\mathcal{O}(\varepsilon).
$$
However, since $f_\varepsilon \in E_k$ we can apply Corollary \ref{cor1} to
conclude that
$$
B_\varepsilon^{t/\varepsilon} f_\varepsilon = (1 - \varepsilon \Delta
f_\varepsilon + o(\varepsilon) )^{t/\varepsilon} = (1 - \varepsilon
\Delta)^{t/\varepsilon} f_\varepsilon + o(\varepsilon).
$$
where, in a slight abuse of notation, we have used $o(\varepsilon)$ to denote a
function whose $L^2$ norm divided by $\varepsilon$ goes to zero as $\varepsilon
\to 0$. Then observing that
$$
\| (1 - \varepsilon \Delta)^{t/\varepsilon} f_\varepsilon - e^{-t \Delta}
f_\varepsilon \|_{L^2} \rightarrow 0 \quad \text{as} \quad \varepsilon
\rightarrow
0,
$$
completes the proof.
\end{proof}

\section{Proof of Theorem \ref{thm2}} \label{proofthm2}

\subsection{Notation} 
In addition to the notation defined in Section \ref{notation1}, recall that
$d\nu(r) = p(r) dr$ is a given reference measure, where $p \in C^3(X)$ is a
positive density function with respect to the volume measure $dr$ on $X$, and
furthermore, recall that $d\hat{\nu}(r) = d\nu(r)/v(r)$ where $v \in C^3(X)$ is
a given positive weight function. 

\subsection{Existence of a bi-stochastic normalization}
The existence of a positive function $d \in C^\infty(X)$ such that the
kernel $c_\varepsilon : X \times X \rightarrow \mathbb{R}$ defined by 
$$
c_\varepsilon(x,y) = \frac{\int_X k_\varepsilon(x,r) k_\varepsilon(y,r)
d\hat{\nu}(r)}{d(x) d(y)} \quad \text{is bi-stochastic with respect to}
\quad
d\hat{\mu},
$$
follows from the same argument as in Section
\ref{existence}. 

\subsection{Intermediate results}
The proof of Theorem \ref{thm2} follows the same pattern as the proof of Theorem
\ref{thm1}. First, we establish expansions of $C_\varepsilon$ near and away from
the boundary, and then combine these expansions to complete the proof.

\begin{proposition} \label{prop3}
Suppose that $0 < \gamma < 1/2$ and $f \in C^3(X)$. If $x \in X$ is at least
distance $2 \varepsilon^\gamma$ from $\partial X$, then
$$
C_\varepsilon f = f - \frac{\varepsilon m_2}{2 m_0} \left( 
\frac{ \Delta \left( f  \left( \frac{q \cdot w}{v \cdot p} \right)^{1/2} \right)}{\left( \frac{q \cdot
w}{v \cdot p} \right)^{1/2}}
-
\frac{ \Delta \left( \frac{q \cdot w}{v \cdot p} \right)^{1/2} }{\left( \frac{q
\cdot w}{v \cdot p} \right)^{1/2}} f
+
\frac{ \Delta \left( f  \left( \frac{q \cdot p}{v \cdot w} \right)^{1/2}
\right)}{\left( \frac{q \cdot p}{v \cdot w} \right)^{1/2}}
-
\frac{ \Delta \left( \frac{q \cdot p}{v \cdot w} \right)^{1/2}
}{\left( \frac{q \cdot p}{v \cdot w} \right)^{1/2}}  f
\right) + \mathcal{O} \left( \varepsilon^2 \right).
$$
\end{proposition}
\begin{proof}
Since the kernel $c_\varepsilon$ is bi-stochastic with respect to $d\hat{\mu}$
we have the following integral equation
$$
1 = \int_X c_\varepsilon(x,y) d\hat{\mu}(y) = \int_X \frac{\int_X
k_\varepsilon(x,r) k_\varepsilon(y,r) \frac{p(r)}{v(r)} dr}{d(x) d(y)}
\frac{q(y)}{w(y)} dy.
$$
Exchanging the order of integration by Fubini's Theorem gives
$$
1 = \int_X \frac{k_\varepsilon(x,r) p(r)}{d(x) v(r)} \int_X
k_\varepsilon(y,r)   \frac{q(y)}{d(y) w(y)} dy dr.
$$
Since $k_\varepsilon(x,y) = h(|x-y|^2/\varepsilon)$ where $h$ has exponential
decay, the outside integral can be truncated to a ball of radius
$\varepsilon^\gamma$ around $x$ while incurring exponentially small error. In
particular, we have
$$
1 = \int_{\{r \in X : |r - x| < \varepsilon^\gamma\}} \frac{k_\varepsilon(x,r)
p(r)}{d(x) v(r)} \int_X k_\varepsilon(y,r)   \frac{q(y)}{d(y) w(y)} dy dr \left(1 +
\mathcal{O}\left(\varepsilon^2 \right) \right).
$$
Since we assumed that $d(x,\partial X) \le 2 \varepsilon^\gamma$ it follows
that $d(r,\partial X) \le \varepsilon^\gamma$ for all $r$ in the above integral.
\begin{figure}[h!]
\centering
\includegraphics[width=0.4\textwidth]{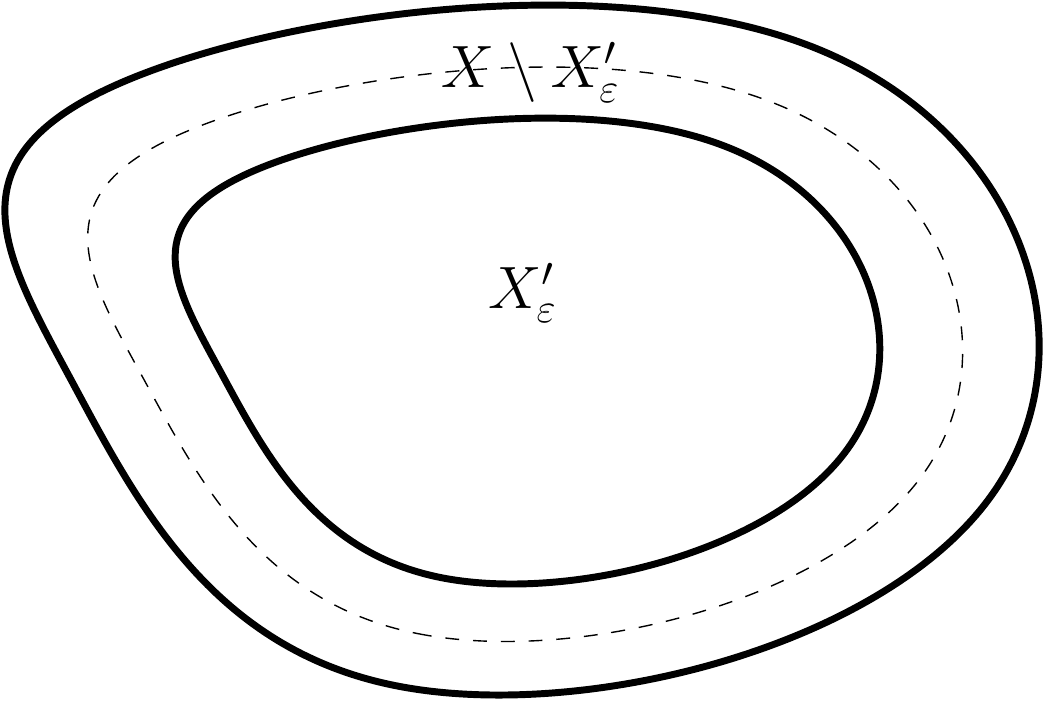}
\caption{We partition $X = X_\varepsilon^\prime \cup (X \setminus
X_\varepsilon^\prime)$ where $X_\varepsilon^\prime = \{x \in X : d(x,\partial X)
> 2 \varepsilon^\gamma\}$; the dotted line represents the fact that $\{ y \in X
: |x^\prime - y| < \varepsilon^\gamma \} \subset \{ y \in X : d(y,\partial X) >
\varepsilon^\gamma \}$ for each $x^\prime \in X_\varepsilon^\prime$.
}
\end{figure}
Therefore, we may apply Lemma \ref{lem1} to the inner integral to conclude that
$$
1 = \int_{\{r \in X : |r - x| < \varepsilon^\gamma\}} \frac{k_\varepsilon(x,r)
p(r)}{d(x) v(r)} \varepsilon^{d/2} \frac{m_0 q(r)}{d(r)w(r)} \left(1 +
\frac{\varepsilon m_2}{2 m_0} \left( E(x) - \frac{\Delta \left( \frac{q(r)}{d(r)
w(r)} \right)}{\frac{q(r)}{d(r)w(r)}} \right) + \mathcal{O} \left(\varepsilon^2
\right) \right) dr.
$$
Since the restriction of the integral to the ball $\{r \in X : |r-x| <
\varepsilon^\gamma \}$ incurs exponentially small error, we may apply Lemma
\ref{lem1} to the outside integral and rearrange terms to obtain
$$
d^2 = \varepsilon^d m_0^2 \frac{ q \cdot p}{v \cdot w} \left( 1 +
\frac{\varepsilon m_2}{2 m_0} \left(2 E(x) - \frac{\Delta \left( \frac{q}{d\cdot
w}\right)}{\left(\frac{q}{d\cdot w}\right)} - \frac{\Delta \left( \frac{q \cdot
p}{d \cdot v \cdot w}\right)}{\left( \frac{q \cdot p}{d \cdot v \cdot
w}\right)} \right) + \mathcal{O}\left(\varepsilon^2 \right)  \right).
$$
Next, we compute the expansion of $C_\varepsilon$ applied to an arbitrary
function $f \in C^3(X)$ 
$$
C_\varepsilon f (x) = 
\int_X \frac{\int_X k_\varepsilon(x,r) k_\varepsilon(y,r) \frac{p(r)}{v(r)}
dr}{d(x) d(y)} f(y) \frac{q(y)}{w(y)} dy.
$$
Truncating the integral over $r$ and applying Lemma \ref{lem1} twice as above
yields
$$
C_\varepsilon f = \varepsilon^d m_0^2 \frac{f \cdot p \cdot q}{d^2 \cdot v \cdot
w} \left( 1 + \frac{\varepsilon m_2}{2 m_0} \left(2 E(x) - \frac{\Delta \left(
\frac{f\cdot q}{d \cdot w} \right)}{ \left( \frac{f \cdot q}{d \cdot w}
\right)} - \frac{\Delta \left( \frac{f \cdot p \cdot q}{v \cdot w}
\right)}{\left( \frac{f\cdot p \cdot q}{v \cdot w} \right)} \right)
+ \mathcal{O}(\varepsilon^2) \right).  
$$
Substituting the expression for $d^2$ into this expression gives
$$
C_\varepsilon f = f \left( 1 + \frac{\varepsilon m_2}{2 m_0} \left( \frac{\Delta
\left( \frac{f\cdot q}{d \cdot w} \right)}{ \left( \frac{f \cdot q}{d
\cdot w} \right)} + \frac{\Delta \left( \frac{f \cdot p \cdot q}{v \cdot
w} \right)}{\left( \frac{f\cdot p \cdot q}{v \cdot w} \right)}
- \frac{\Delta \left( \frac{q}{d \cdot w}
\right)}{ \left( \frac{q}{d \cdot w} \right)} - \frac{\Delta \left(
\frac{p \cdot q}{v \cdot w} \right)}{\left( \frac{p \cdot q}{v \cdot w}
\right)} \right) \right)
+ \mathcal{O}(\varepsilon^2).
$$
Observe that
$$
d = \varepsilon^{d/2} m_0 \left( \left( \frac{p \cdot q}{v \cdot w}
\right)^{1/2} + \mathcal{O}\left(\varepsilon \right) \right).
$$
Substituting this expression into the above
equation gives 
$$
C_\varepsilon f =  f - \frac{\varepsilon m_2}{2 m_0} \left( 
\frac{ \Delta \left( f  \left( \frac{q \cdot v}{w \cdot p} \right)^{1/2} \right)}{\left( \frac{q \cdot
v}{w \cdot p} \right)^{1/2}}
-
\frac{ \Delta \left( \frac{q \cdot v}{w \cdot p} \right)^{1/2} }{\left( \frac{q
\cdot v}{w \cdot p} \right)^{1/2}} f
+
\frac{ \Delta \left( f  \left( \frac{q \cdot p}{w \cdot v} \right)^{1/2}
\right)}{\left( \frac{q \cdot p}{w \cdot v} \right)^{1/2}}
-
\frac{ \Delta \left( \frac{q \cdot p}{w \cdot v} \right)^{1/2}
}{\left( \frac{q \cdot p}{w \cdot v} \right)^{1/2}}  f
\right) + \mathcal{O} \left( \varepsilon^2 \right),
$$
as was to be shown.
\end{proof}

\begin{proposition} \label{prop4}
Suppose that $0 < \gamma < 1/2$ and $k > 0$ is fixed. If $f \in E_k$, then
uniformly over $x \in X$ within distance $2 \varepsilon^{\gamma}$ of $\partial
X$
$$
C_{\varepsilon} f(x) = f(x_0) + \mathcal{O}(\varepsilon),
$$
where $x_0$ is the closest point in $\partial X$ to $x \in X$ (with respect to
Euclidean distance).
\end{proposition}

\begin{proof}
We begin with the equation
$$
1 = \frac{1}{d(x)} \int_{\{r \in X : |r - x| < \varepsilon^\gamma\}}
k_\varepsilon(x,r) \frac{p(r)}{v(r)} \int_X k_\varepsilon(y,r)
\frac{q(y)}{d(y) w(y)} dy dr.
$$
which was derived in the proof of Proposition \ref{prop3}. Since $d(x,\partial
X) \le \varepsilon^\gamma$, it follows that all $r$ in the above equation
satisfy $d(r,\partial X) < 3 \varepsilon^\gamma$; therefore, the result of
Lemma \ref{lem2} can be applied to the integral over $y$, where the factor $3$
is absorbed into the error term. Thus, we have
$$
\int_X k_\varepsilon(y,r)   \frac{q(y)}{d(y) w(y)} dy =\varepsilon^{d/2}
\frac{q(r_0)m_{0,\varepsilon}(r)}{d(r_0) w(r_0)}  
\left(1 +
\varepsilon^{1/2} \frac{m_{1, \varepsilon}(r)}{m_{0,\varepsilon}(r)} \frac{\partial_n
\frac{q(r_0)}{d(r_0) w(r_0)} }{\frac{q(r_0)}{d(r_0) w(r_0)}}  + \mathcal{O}(\varepsilon)
\right),
$$
where $r_0$ is the closest point on $\partial X$ to $r$. Substituting this
expression into our initial integral equation, expanding the integral by Lemma
\ref{lem2}, and rearranging terms gives
$$
d \cdot d_{x_0} = \varepsilon^d \frac{q_{x_0} \cdot {m_{0,\varepsilon}}_{x_0}
\cdot
m_{0,\varepsilon}
\cdot p_{x_0}}{w_{x_0} \cdot v_{x_0}} 
\cdot \left( 1 + \varepsilon^{1/2}
\frac{{m_{1,\varepsilon}}_{x_0}}{{m_{0,\varepsilon}}_{x_0}} \frac{\partial_n
\frac{q_{x_0}}{d_{x_0} \cdot w_{x_0}}}{\frac{q_{x_0}}{d_{x_0} \cdot w_{x_0}}} 
+ \varepsilon^{1/2} \frac{m_{1,\varepsilon}}{m_{0,\varepsilon}}
\frac{\partial_n \frac{{m_{0,\varepsilon}}_{x_0} \cdot
p_{x_0} \cdot q_{x_0}}{d_{x_0} \cdot w_{x_0} \cdot
v_{x_0}}}{\frac{{m_{0,\varepsilon}}_{x_0} \cdot p_{x_0} \cdot
q_{x_0}}{d_{x_0} \cdot w_{x_0} \cdot v_{x_0}}} +
\mathcal{O}\left(\varepsilon \right) \right),
$$
where the argument $x$ of the functions $d(x)$ and $m_{0,\varepsilon}(x)$ have
been suppressed, and where we write $f_{x_0} = f(x_0)$ for notational brevity.
Moreover, as in Proposition \ref{prop2}, if a similar expansion for
$C_\varepsilon f$ is performed, substituting in this derived value for $d \cdot
d_{x_0}$ yields 
$$ 
C_\varepsilon f = f_{x_0} + \mathcal{O} \left( \varepsilon \right),
$$ 
since the terms of order $\varepsilon^{1/2}$ will cancel as in the proof
of Proposition \ref{prop2} since $f$ can be taken out of normal derivative to
the boundary since it verifies the Neumann condition.

\end{proof}
\subsection{Proof of Theorem \ref{thm2}}
\begin{proof}[Proof of Theorem \ref{thm2}]
The proof of Theorem \ref{thm2} follows from an identical argument as the
proof of Theorem \ref{thm1} where the results of Propositions \ref{prop1} and
\ref{prop2} are replaced with the results of Propositions \ref{prop3} and
\ref{prop4}.
\end{proof}

\section{Proof of Propositions \ref{grad1} and \ref{grad2}}
\subsection{Proof of Proposition \ref{grad1}}
\begin{proof}[Proof of Proposition \ref{grad1}]
First using a Nytr\"om type extension we write
$$
\nabla_x \varphi_k(x) = \frac{1}{\lambda_k} \int_X \nabla_x b_\varepsilon(x,y) \varphi_k(y)
d\hat{\mu}(y).
$$
The gradient (in the $x$-variable) of the kernel $b_\varepsilon(x,y)$ can be computed
by the product rule
$$
\nabla_x b_\varepsilon(x,y) = \nabla_x \frac{k_\varepsilon(x,y)}{d(x)d(y)} = \frac{\nabla_x
k_\varepsilon(x,y)}{d(x) d(y)} - \frac{k_\varepsilon(x,y)}{d(x) d(y)}
\frac{\nabla_x d(x)}{d(x)}.
$$
The gradient $\nabla_x k_\varepsilon(x,y)$ can be computed directly
$$
\nabla_x k_\varepsilon(x,y) = -\frac{x-y}{\varepsilon} e^{-|x-y|^2/\varepsilon}
= \frac{y-x}{\varepsilon} k_\varepsilon(x,y).
$$
In order to compute the logarithmic derivative $\nabla_x d(x)/d(x)$ we use the
fact that
$$
d(x) = \int_X \frac{k_\varepsilon(x,y)}{d(y)} d\hat{\mu}(y),
$$
which is an immediate consequence of the bi-stochasticity of $b_\varepsilon$
with respect to $d\hat{\mu}$.  Specifically,
$$
\frac{\nabla_x d(x)}{d(x)} = \frac{1}{d(x)} \int_X \frac{\nabla_x
k_\varepsilon(x,y)}{d(y)} d\hat{\mu}(y) = \frac{1}{d(x)} \int_X \frac{y -
x}{\varepsilon} \frac{k_\varepsilon(x,y)}{d(y)} d\hat{\mu}(y) = \frac{\bar{x} -
x}{\varepsilon},
$$
where
$$
\bar{x} := \int_X y b_\varepsilon(x,y) d\hat{\mu}(y).
$$
Substituting the expressions for $\nabla_x k_\varepsilon(x,y)$ and $\nabla_x
d(x)/d(x)$ into the Nystr\"om extension yields
$$
\nabla_x \varphi_k(x) = \frac{1}{\lambda_k} \int_X \frac{y -
\bar{x}}{\varepsilon} b_\varepsilon(x,y) \varphi_k(y) d\hat{\mu}(y)
$$
which completes the proof.
\end{proof}

\begin{remark}
If $f \in L^2(X,d\hat{\mu})$, then 
$$
\nabla_x f(x) = \sum_{k} \langle f, \varphi_k \rangle_{L^2(X,d\hat{\mu})}
\frac{1}{\lambda_k} \int_X \frac{y - \bar{x}}{\varepsilon} b_\varepsilon(x,y) \varphi_k(y)
d\hat{\mu}(y).
$$
\end{remark}
\noindent However, from a numerical point of view this formula only really
makes sense when the sum is restricted to the first several eigenfunctions, and
$f$ is well approximated by linear combinations of these functions; otherwise,
any error in the computation of $\langle f,\varphi_k
\rangle_{L^2(X,d\hat{\mu})}$ will be amplified by a factor of $1/\lambda_k$.

\subsection{Proof of Proposition \ref{grad2}}
\begin{proof}[Proof of Proposition \ref{grad2}]
The proof follows the same pattern as the proof of Proposition \ref{grad1}.
Using a Nyst\"om like extension we write
$$
\nabla_x \varphi_k(x) = \frac{1}{\lambda_k} \int_X \nabla_x
\frac{c_\varepsilon(x,y)}{d(x) d(y)} \varphi_k(y) d\hat{\mu}(y).
$$
Now by the product rule for differentiation
$$
\nabla_x \frac{c_\varepsilon(x,y)}{d(x) d(y)}
= \frac{\nabla_x c_\varepsilon(x,y)}{d(x) d(y)} - \frac{c_\varepsilon(x,y)}{d(x)
d(y)} \frac{\nabla_x d(x)}{d(x)}.
$$
First, we compute
$$
\nabla_x c_\varepsilon(x,y) = \int_X \nabla_x k_\varepsilon(x,r)
k_\varepsilon(y,r) d \hat{\nu}(r)
= \int_X \frac{r - x}{\varepsilon} k_\varepsilon(x,r) k_\varepsilon(y,r)
d\hat{\nu}(r)
$$
Second, we compute
$$
\frac{\nabla_x d(x)}{d(x)} = \frac{1}{d(x)} \nabla_x  \int_X
\frac{c_\varepsilon(x,y)}{d(y)} d\hat{\mu}(y) 
=
\int_X \frac{\int_X \frac{r - x}{\varepsilon} k_\varepsilon(x,r)
k_\varepsilon(y,r) d\hat{\nu}(r)}{d(x) d(y)} d\hat{\mu}(y).
$$
We assert that
$$
\nabla_x \frac{c_\varepsilon(x,y)}{d(x) d(y)} = \frac{\int_X
\frac{r}{\varepsilon} k_\varepsilon(x,r) k_\varepsilon(y,r)
d\hat{\nu}(r)}{d(x) d(y)} - \frac{c_\varepsilon(x,y)}{d(x) d(y)} \int_X
\frac{\int_X \frac{r}{\varepsilon} k_\varepsilon(x,r) k_\varepsilon(y,r)
d\hat{\nu}(r)}{d(x) d(y)} d\hat{\mu}(y);
$$
indeed, the terms associated with $x/\varepsilon$ are equal except for
having opposite signs so they cancel.  Define 
$$
\quad \bar{r}_x := \int_X r (F_\varepsilon 1) (x,r)d\hat{\nu}(r) \quad
\text{and} \quad (F_\varepsilon f) (x,r) := \int_X \frac{k_\varepsilon(x,r)
k_\varepsilon(y,r)}{d(x) d(y)} f(y) d\hat{\mu}(y).
$$
With this notation
$$
\nabla_x \frac{c_\varepsilon(x,y)}{d(x) d(y)} = \frac{\int_X \frac{r -
\bar{r}_x}{\varepsilon} k_\varepsilon(x,r) k_\varepsilon(y,r)
d\hat{\nu}(r)}{d(x) d(y)}.
$$
Therefore,
$$
\nabla_x \varphi_k(x) = \frac{1}{\lambda_k} \int_X \frac{r -
\bar{r}_x}{\varepsilon} (F_\varepsilon \varphi_k)(x,r) d\hat{\nu}(r),
$$
as was to be shown.

\end{proof}

\end{document}